\documentclass[journal,10pt]{IEEEtran}
\usepackage{graphicx}
\usepackage[colorlinks,linkcolor=black,anchorcolor=black,citecolor=black]{hyperref}
\usepackage{lineno,hyperref}
\modulolinenumbers[5]
\usepackage[cmex10]{amsmath}
\usepackage{amsfonts}

\usepackage{color}
\usepackage{algorithm}
\usepackage{algorithmic}
\usepackage{array}
\usepackage{multirow}
\usepackage{booktabs}
\usepackage{bm}

\allowdisplaybreaks[4]
\usepackage{amsmath}
\usepackage{underscore}
\usepackage{stfloats}
\usepackage{float}
\usepackage{lipsum}
\usepackage{cite}
\usepackage{subfigure}
\usepackage{adjustbox}
\usepackage{amsthm}
\usepackage{xcolor}
\usepackage{makecell}

\newtheorem{theorem}{Theorem}
\newtheorem{definition}{Definition}

\makeatletter

\usepackage{amssymb}
\usepackage{amsmath}
\begin{document}

\title{World Model-Enabled Causal Digital Twins for Semantic Communications in Physical AI Systems}

\author{Lingyi Wang,~\IEEEmembership{Graduate Student Member,~IEEE,}
Tingyu Shui,~\IEEEmembership{Graduate Student Member,~IEEE,}\\
Walid Saad,~\IEEEmembership{Fellow,~IEEE,}
and Pascal Adjakple

% \thanks{This research was supported by the U_t.S. National Science Foundation under Grant CNS-2225511.}
% , the National Natural Science Foundation of China under Grant 62271267 
% and the open research fund of National Mobile Communications Research Laboratory, Southeast University under Grant 2024D16.}
\thanks{Lingyi Wang, Tingyu Shui, and Walid Saad are with the Bradley Department of Electrical and Computer Engineering, Virginia Tech, Alexandria, VA, 22305, USA.
(emails: \{lingyiwang, tygrady, walids\}@vt.edu).}
\thanks{
Pascal Adjakple is with the Research and Innovation Labs, InterDigital Communications Inc., New York, NY 10120, USA
(email: Pascal.Adjakple@InterDigital.com).}
}

\maketitle

\vspace{-1cm}
\begin{abstract} 
Semantic communication has emerged as a promising paradigm for enabling goal-oriented networking. However, most existing semantic communication solutions are tailored to one-shot tasks and optimize instantaneous performance. Hence, they cannot be used to support closed-loop dynamic systems with physical artificial intelligence (AI), in which the transmitted semantics affect not only the current inference outcome but also future control actions, state evolution, and ultimately long-horizon task performance. To address this gap, this paper investigates goal-oriented semantic communications for physical AI systems with closed-loop sensing-communication-inference-control. In particular, the problem of semantic communications is formulated as a long-term return-per-bit maximization under wireless bit-budget constraints while capturing both control efficiency and communication efficiency. To solve this problem, a novel causal information value (CIV) metric is introduced to evaluate the marginal contribution of each semantic token to the expected long-term return by transmission interventions. Then, a world-model-enabled causal digital twin (WM-CDT) framework is proposed to capture the dynamics of closed-loop physical AI systems and enable counterfactual reasoning for long-horizon imagined rollouts. Based on these imagined rollouts, an actor-critic policy is trained for long-horizon agent control with high data efficiency, while the semantic token selector is trained through CIV-per-bit evaluation. Extensive simulations on an AirSim-Sionna-based unmanned aerial vehicle (UAV) navigation simulator with realistic wireless channel modeling show that the proposed WM-CDT framework achieves significant improvement in return-per-kbit and navigation success rate compared to existing reinforcement learning solutions. Moreover, CIV is shown to provide a better estimation of the true long-horizon return gain compared to existing information value metrics.
\end{abstract}

% keywords

\IEEEpeerreviewmaketitle

% \begin{figure*}
%     \centering
%     \includegraphics[width=1\linewidth]{framework.png}
%     \vspace{-0.9cm}
%     \caption{The framework of the proposed IV-driven semantic communications.}
%     \vspace{-0.4cm}
%     \label{fig:framework}
% \end{figure*}

\section{Introduction}
Recent advances in embodied and physical artificial intelligence (AI) are driving wireless networks toward integrated sensing, communication, inference, and control architectures~\cite{saad2024artificial,11016266,wang2025bridging}. In such systems, physical agents, such as unmanned aerial vehicles (UAVs), autonomous vehicles, and mobile robots, interact with dynamic environments, other agents, and edge servers through sequential sensing, communication, and control decisions~\cite{wang2025bridging}. The networked deployment of such physical AI agents creates a fundamental challenge for wireless system design, as the network must support not only reliable information delivery, but also timely and task-relevant interactions with the physical world over extended horizons. Semantic communication~\cite{a2,a1,a3} provides a promising approach for addressing this challenge by shifting wireless system design beyond reliable bit delivery toward task-oriented information exchange in which the goal is not to reproduce the source signal itself, but to convey the information that is useful for accomplishing a downstream task. This property is particularly important for closed-loop physical AI systems, where high-dimensional sensory observations must be converted into compact and task-relevant representations that not only support timely edge inference and control, but also influence future state evolution and long-horizon task performance over bandwidth-constrained wireless links~\cite{a10}. Deploying goal-oriented semantic communication for closed-loop physical AI systems yields two fundamental challenges: (a) how to evaluate the value of semantic information, and (b) how to learn and optimize long-horizon task performance under communication constraints. 

For addressing Challenge (a), existing importance-aware semantic communication solutions~\cite{a12} usually rely on feature importance scores or instantaneous task losses. Although such metrics are useful for one-step inference, in the context of physical AI, they cannot capture the delayed impact of semantic transmission on future control actions, state evolution, and long-horizon task utility. Traditional information-centric metrics, such as age of information (AoI)~\cite{r9} and value of information (VoI)~\cite{r10,r11}, measure information freshness or relevance, but they do not fully characterize the closed-loop coupling among sensing, communication, inference, and control. In particular, the observed correlation between transmitted semantics and high task return does not necessarily represent the true marginal contribution of the transmitted information, since the transmission decision itself depends on the system history. Hence, instead of only observing whether the transmitted semantic information is correlated with a high task return, we need to ask a \emph{counterfactual} question: given the same past system history, what would the future trajectory and long-term return have been if this semantic information had not been transmitted? By using the concept of \emph{counterfactual reasoning} that compares this hypothetical trajectory with the actual one in which the semantic information is transmitted, the system can isolate the closed-loop causal effect of that information on future belief updates, control actions, state evolution, and long-horizon task performance. This motivates defining a causal long-horizon information value metric that goes beyond statistical association.

When it comes to Challenge (b), despite the success of traditional mathematical optimization approaches~\cite{11164706,10943268,10974474,11069265}, such as dynamic programming~\cite{11164706}, Lyapunov optimization~\cite{10943268}, and model predictive control~\cite{11069265}, their direct deployment in closed-loop physical AI systems is challenging due to high-dimensional state spaces, uncertain and time-varying dynamics, and the need for long-horizon task-aware decision making. To alleviate these challenges, there has been a recent surge of reinforcement learning (RL)-driven solutions~\cite{10758370,10744542,10915558,10077432}, which can learn directly from sensory data and wireless data, and adapt to dynamic environments without predefined models. However, existing RL-based approaches~\cite{10758370,10744542,10915558,10077432} are not suitable for closed-loop physical AI systems. First, RL relies on trial-and-error interactions, which are data-inefficient and may be unsafe in physical AI systems. Second, RL-based approaches mainly optimize policies from current observations, rather than explicitly reasoning over future wireless and physical states. Hence, they cannot fully address the long-horizon optimization requirements of closed-loop physical AI systems. These limitations call for a new learning framework that allows a wireless network to support physical AI by endowing it with the ability to model closed-loop dynamics, evaluate counterfactual semantic value, and plan over future system evolutions under communication constraints.

The main contribution of this paper is to address the above challenges by developing a novel world-model-enabled causal digital twin (WM-CDT) framework that enables goal-oriented semantic communications in closed-loop physical AI systems. In the proposed framework, the semantic communication problem is formulated as a long-term return-per-bit maximization problem that jointly captures control efficiency and communication efficiency. To quantify the long-horizon utility of transmitted semantics, we introduce a causal information value (CIV) metric that evaluates the causal effect of semantic tokens on the expected long-term return through transmission interventions. We then develop a WM-CDT to learn the closed-loop dynamics of physical AI systems, support counterfactual reasoning for CIV estimation, and enable long-horizon action planning. In summary, our key contributions include:
\vspace{-0.15cm}
\begin{itemize}
\item We study goal-oriented semantic communications for closed-loop physical AI systems with integrated sensing, communication, inference, and control. Semantic communication allows the proposed system to convey compact task-related information for optimizing long-horizon system performance. To capture the tradeoff between long-horizon task performance and communication efficiency, we formulate the semantic communication problem as a return-per-bit maximization approach that jointly optimizes semantic token selection and agent control.

    \item We introduce a novel CIV metric to capture long-horizon causality between transmitted semantic tokens and long-horizon task returns in closed-loop physical AI systems. Particularly, CIV quantifies the marginal long-horizon contribution of each semantic token by comparing the expected returns under two counterfactual interventions that share the same system history but differ only in whether the token is transmitted.
	
	\item We develop a WM-CDT framework to solve the return-per-bit maximization problem over extended horizons. At the edge server, the digital twin learns the closed-loop system dynamics from received semantic tokens and executed control actions, and enables long-horizon imagined rollouts. Based on this imagination ability, an actor-critic (AC)-based control policy is trained for long-horizon agent control with high data efficiency, while the semantic token selector is trained through CIV-per-bit evaluation.

    \item We introduce an AirSim-Sionna-based simulator with realistic wireless channel modeling for physical AI systems. Extensive simulation results on the UAV navigation task show that WM-CDT achieves up to $17.3\%$, $55.4\%$, $30.7\%$, and $33.7\%$ higher return-per-kbit than AC-based recurrent reinforcement learning (AC-RRL) \cite{gc}, model-based policy optimization (MBPO), AC, and proximal policy optimization (PPO), respectively, and improves the navigation success rate by up to $9.6\%$, $26.3\%$, $16.1\%$, and $18.0\%$, respectively. Moreover, CIV is shown to be more aligned with the true long-horizon return gain than the myopic VoI, saliency score, and confidence score.
\end{itemize}

\vspace{-0.4cm}
\section{Related Works}
\vspace{-0.2cm}
\subsection{Goal-Oriented Semantic Communications}
\vspace{-0.15cm}
Existing research on goal-oriented semantic communications~\cite{r3,r4,r5,r6,r7} has mainly focused on one-shot tasks, such as image classification~\cite{r3} and speech-to-text translation~\cite{r4}. The solutions in ~\cite{r3,r4,r5,r6,r7} seek to transmit task-relevant semantics for improving the performance of a specific downstream task, but they do not consider long-horizon closed-loop system-level optimization. The work in~\cite{r3} proposed a multi-task semantic communication framework for autonomous vehicles that jointly performs image reconstruction and classification. Recent studies~\cite{r5,r6,r7} have also moved toward data-efficient, channel-aware, and explicitly goal-driven semantic communication designs. For instance, the authors in~\cite{r5} proposed a self-supervised low-label semantic communication framework to reduce the dependence on annotated data. Despite these advances, most existing works~\cite{r3,r4,r5,r6,r7} still optimize semantic transmission for instantaneous or short-term task performance. They do not explicitly account for how transmitted semantics affect future belief updates, control actions, state evolution, and long-horizon task performance in closed-loop physical AI systems. In \cite{gc}, we developed a semantic communication framework for supporting physical AI agents. However, our solution in \cite{gc} does not take into account the causal effects of transmitted semantics, thus it cannot capture long-horizon information value and lacks long-term closed-loop optimization in physical AI systems.

\vspace{-0.4cm}
\subsection{Information Value in Communication and Control}
\vspace{-0.1cm}
Research on information value in communication and control~\cite{r9,r10,r11} ranges from freshness-aware scheduling toward task-aware and semantics-aware metrics. The authors in~\cite{r9} compared AoI and VoI in cellular networked control systems and showed that information freshness alone cannot capture control relevance. More recent studies~\cite{r12,r14,r15} have extended these metrics toward semantic importance and joint communication-control optimization. In~\cite{r12}, the authors developed a semantic value metric to measure the importance of semantic information in text transmission, and in~\cite{r14} and~\cite{r15}, VoI was exploited for the joint optimization of control and communication. However, the works in~\cite{r9,r10,r11} and \cite{r12,r14,r15} mainly evaluate information utility for short-term scheduling objectives, and they do not explicitly capture the influence of transmitted semantics on long-horizon task return in closed-loop physical AI systems, where transmitted information affects future belief updates, control decisions, state evolution, and environment-agent dynamics.

\vspace{-0.4cm}
\subsection{World Model-Based Learning Approaches}
\vspace{-0.1cm}
The use of world models for long-horizon planning in robotics and embodied AI tasks has recently attracted attention \cite{r21,r22,r24}. In~\cite{r18,r17,r19}, the authors showed that recurrent state-space model (RSSM)-based world models can support robust imagination-based reinforcement learning across diverse domains. 
The works in~\cite{r25,r26,r27,r28,r29} used world models in the context of wireless networking and edge intelligence. In~\cite{r25} and~\cite{r26}, we proposed world-model-based learning approaches for highly dynamic and mobile wireless networks with few-shot generalization. These works~\cite{r25,r26,r27,r28,r29} demonstrate the promise of world models for improving data efficiency, generalization, and long-horizon optimization in wireless systems. 
However, existing wireless world-model approaches~\cite{r25,r26,r27,r28,r29} represent network-side states and actions for network evolution prediction, rather than modeling information as intervention variables that update the remote controller's belief and influence future physical actions. Hence, these approaches cannot evaluate long-horizon information value.
This motivates a world-model design that explicitly couples semantic transmission, belief evolution, physical agent control, and long-term task return in physical AI systems.

\begin{figure*}[ht!]
  \centering
  \includegraphics[width=0.7\linewidth]{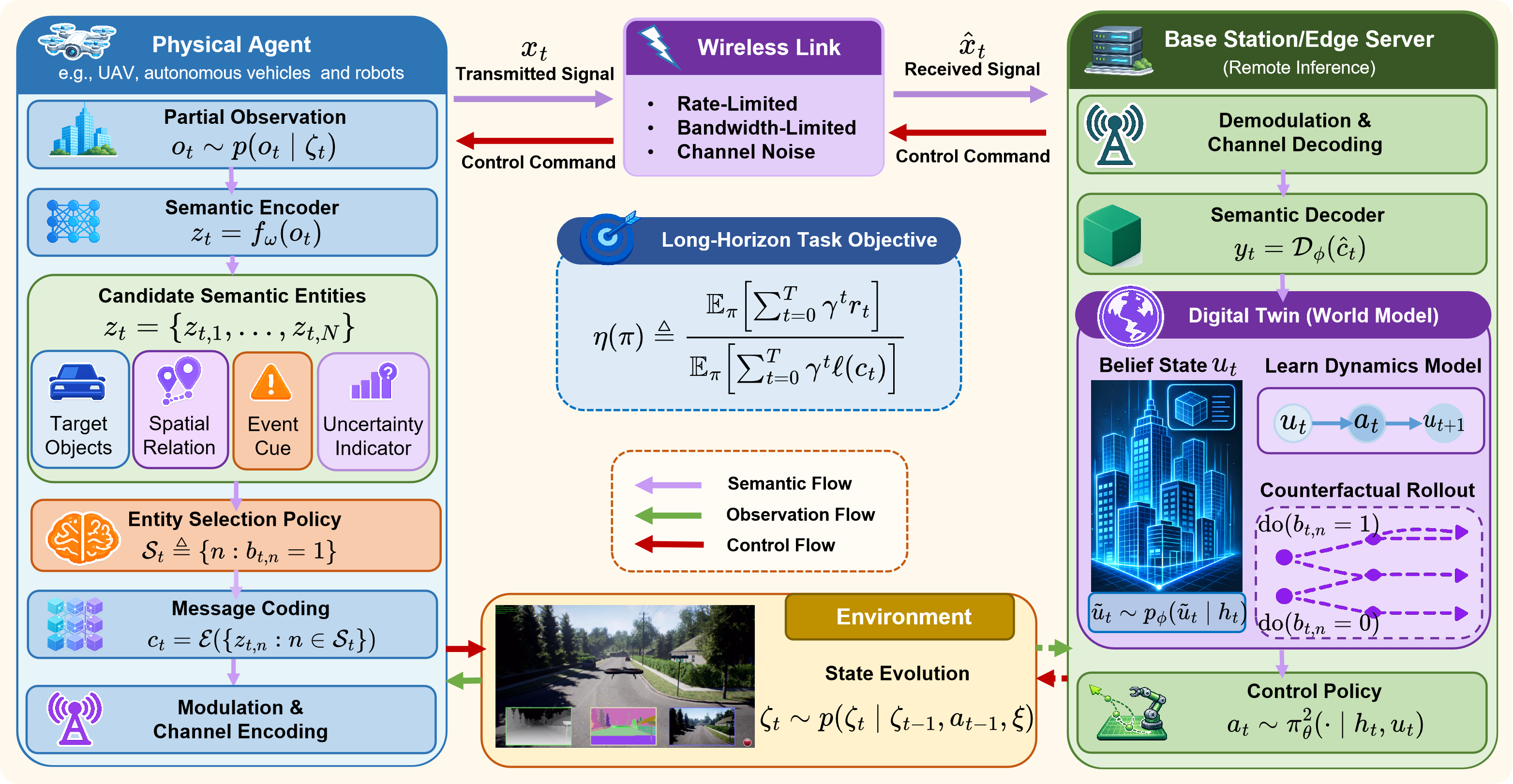}
  \vspace{-0.35cm}
  \caption{The closed-loop goal-oriented semantic communication system with a world model-enabled causal digital twin.}
  \vspace{-0.7cm}
  \label{fig:system}
\end{figure*}

\vspace{-0.4cm}
\section{System Model for Goal-Oriented Semantic Communications}
\vspace{-0.2cm}
In this section, we introduce goal-oriented semantic communications in closed-loop physical AI systems. Then, we formulate a system-level return-per-bit optimization problem by jointly optimizing the agent control and semantic transmission.

\vspace{-0.4cm}
\subsection{System Model} 
\vspace{-0.1cm}
We consider a closed-loop wireless physical AI system consisting of a physical AI agent, such as a UAV or a mobile robot, and a remote decision maker deployed at an edge server, as shown in Fig. \ref{fig:system}. The physical AI agent senses and acts in a physical environment, whose state at time slot $t$ is represented by $\boldsymbol{\zeta}_t$. After the execution of control action $\boldsymbol{a}_{t-1}$, the environment state evolves as
$
    \boldsymbol{\zeta}_t \sim p(\boldsymbol{\zeta}_t \mid \boldsymbol{\zeta}_{t-1}, \boldsymbol{a}_{t-1}, \xi),
$
where $\xi$ represents exogenous uncertainty and disturbances, and $\boldsymbol{a}_{t-1}$ is the action of the physical AI agent at time $t-1$. From the environment dynamics, the physical AI agent acquires a local partial observation $\boldsymbol{o}_t\sim p(\boldsymbol{o}_t\mid\boldsymbol{\zeta}_t)$.

At each time slot $t$, the physical agent first executes the control command $\boldsymbol{a}_{t-1}$ resulting from the control policy computed at the edge server and then extracts $N$ candidate semantic-entity tokens (referred to hereinafter as ``semantic tokens'' for brevity) from the acquired local observation $\boldsymbol{o}_t$ by using a semantic encoder $f_{\omega}(\cdot)$:
\vspace{-0.3cm}
\begin{equation}
    \boldsymbol{z}_t = f_{\omega}(\boldsymbol{o}_t)
    = \{\boldsymbol{z}_{t,1},\ldots,\boldsymbol{z}_{t,N}\},
\end{equation}
where $\omega$ captures the parameters of the encoder at the physical agent.
Each semantic token $\boldsymbol{z}_{t,n}$ is a compact symbolic descriptor of a task-relevant aspect of the current observation, such as an object hypothesis, a spatial relation, an event cue such as an approaching obstacle or collision risk, or an uncertainty indicator. Since the wireless network's uplink is rate-limited, the agent cannot transmit all candidate tokens at every time slot. We use $b_{t,n}\in\{0,1\}$ to indicate whether token $n$ is transmitted, and the token selection is generated by the selector $b_{t,n} \sim \pi^1_{\varpi}$. Let the selected token subset be $\mathcal{B}_t\triangleq\{n:b_{t,n}=1\}$. The selected tokens are encoded into a semantic message by
$
    \boldsymbol{c}_t = \mathcal{E}\!\left(\{\boldsymbol{z}_{t,n}: n \in \mathcal{B}_t\}\right),
$
with code length $\ell(\boldsymbol{c}_t)$ bits, and then modulated into a physical transmit signal $x_t$ for uplink delivery.

The received signal at the remote server is given by
$
\hat{\boldsymbol{x}}_t = \boldsymbol{H}_t \boldsymbol{x}_t + \boldsymbol{n}_t,
$
where $\boldsymbol{H}_t$ is the time-varying wireless channel and $\boldsymbol{n}_t$ is the additive noise. The received signal is then demodulated and decoded into $\hat{\boldsymbol{c}}_t$, from which semantic tokens $\boldsymbol{y}_t$ are reconstructed by the message decoder $\boldsymbol{y}_t = \mathcal{Q}(\hat{\boldsymbol{c}}_t)$.
Based on the received tokens, the edge server maintains a belief $\boldsymbol{u}_t$ of the agent system by the recurrent model:
\vspace{-0.25cm}
\begin{subequations}
	\begin{align}
    \boldsymbol{h}_t &= \Psi_{\phi}(\boldsymbol{h}_{t-1}, \boldsymbol{u}_{t-1}, \boldsymbol{a}_{t-1}),\\
	\boldsymbol{u}_t &\sim q_{\phi}(\boldsymbol{u}_t \mid \boldsymbol{h}_{t}, \boldsymbol{y}_t),
	\end{align}
\end{subequations}
where $\boldsymbol{h}_t$ is the historical deterministic information.
Then, the edge server selects the next agent control with the policy $\pi^2_{\theta}$ by
$
    \boldsymbol{a}_t \sim \pi^2_{\theta}(\cdot \mid \boldsymbol{h}_t, \boldsymbol{u}_t),
$
and feeds the control command $\boldsymbol{a}_t$ to the physical agent, where $\theta$ is the parameter of the control policy.  
In contrast to existing semantic communication solutions~\cite{r3,r4,r5,r6,r7} that are tailored to one-step downstream tasks, in our model, the selected semantic tokens $\boldsymbol{y}_t$ from the physical agent to the remote server shape the recurrent control state $\boldsymbol{h}_t$, which, in turn, impacts the subsequent control commands $\boldsymbol{a}_{>t}$ taken by the agent. Moreover, these control commands further reshape future environment evolution $\boldsymbol{\zeta}_{>t}$, observations $\boldsymbol{o}_{>t}$, and long-term task performance. Hence, most existing one-step optimization solutions~\cite{r3,r4,r5,r6,r7} are inadequate for capturing this closed-loop physical AI system, and a new system-level formulation is required to jointly optimize long-horizon communication and task performance.

\vspace{-0.5cm}
\subsection{Problem Formulation}
\vspace{-0.1cm}
We define $r_t$ as the task reward at time slot $t$.
In a networked physical AI system, $r_t$ represents the closed-loop task outcome induced by the transmitted semantics and the resulting control action. 
For example, in UAV navigation, $r_t$ can reward the UAV for moving closer to the destination and successfully reaching the target, while penalizing collisions, unsafe trajectories, or unnecessary motion. Similarly, in autonomous driving, $r_t$ can reward the vehicle for safely following the planned route, maintaining lane discipline, and keeping a safe distance from surrounding vehicles, while penalizing collisions, lane departures, abrupt maneuvers, or inefficient driving.
Thus, a semantic message is valuable not because it is accurately reconstructed, but because it helps the remote server make actions that improve the task performance over time. The long-horizon discounted return from time $t$ onward is defined as
$G_t \triangleq \sum_{\tau=t}^{T}\gamma^{\tau-t}r_{\tau}$,
where $\gamma\in(0,1]$ is a discount factor, and $T$ is the episode horizon.
Under the joint semantic selection and control policy $\pi=\{\pi^1_{\varpi},\pi^2_{\theta}\}$, the discounted task return and discounted communication cost are, respectively, given by
\vspace{-0.3cm}
\begin{equation}
R(\pi)\!\triangleq\!\mathbb{E}_{\pi}\!\!\left[\sum_{t=0}^{T}\!\gamma^tr_t\!\right]\!\!=\!\mathbb{E}_{\pi}\!\left[G_0\right],
C(\pi)\!\triangleq\!\mathbb{E}_{\pi}\!\left[\sum_{t=0}^{T}\!\!\gamma^t\ell(\boldsymbol{c}_t)\!\right].
\end{equation}
Here, $R(\pi)$ measures the expected long-horizon physical task performance, while $C(\pi)$ measures the expected number of semantic bits consumed to support the closed-loop interaction.
Hence, the semantic communication problem for the closed-loop physical AI system is formulated as
\vspace{-0.15cm}
\begin{subequations}\label{eq:P1_ratio_obj_bit}
\begin{align}
\max_{\pi}\quad & \eta(\pi) \triangleq \frac{R(\pi)}{C(\pi)}\\
\text{s.t.}\quad & C(\pi)>0, \\
& \ell(\boldsymbol{c}_t)\le U_t, \quad \forall t,
\end{align}
\end{subequations}
where $\eta(\pi)$ measures the amount of long-horizon physical task utility obtained per transmitted semantic bit, and $U_t$ is the per-slot uplink bit budget.
The ratio $\eta(\pi) $ is important for semantic communications in closed-loop physical AI systems. For instance, in UAV navigation, a policy has a high $\eta(\pi)$ if it reaches the destination and avoids collisions by transmitting only a small number of control-relevant semantic descriptions, such as obstacle locations or goal-direction cues. 
In contrast, a policy has a low $\eta(\pi)$ if it transmits many visually detailed but control-irrelevant features that do not improve the UAV trajectory, or if it saves bits so aggressively that the remote controller loses critical obstacle information and the UAV fails the task.
Therefore, maximizing $\eta(\pi)$ encourages the communication system to allocate bits to semantics that improve closed-loop system-level performance, rather than to semantics that are merely informative for one-step perception.

The optimization problem in \eqref{eq:P1_ratio_obj_bit} is challenging to solve by existing RL approaches~\cite{10758370,10744542,10915558,10077432} for three main reasons. First, the utility of a transmitted token must be evaluated over extended horizons. Once received by the edge server, the token changes the latent belief of the environment-agent system, which further impacts the selected control action, the subsequent physical state evolution, future observations, and ultimately the long-horizon return. Hence, the value of a semantic token cannot be evaluated from the current time slot alone or from its observed correlation with a high task return. 
Second, semantic token selection is a combinatorial decision tightly coupled with physical agent control under the per-slot bit-budget constraint. The contribution of each token depends on the other transmitted tokens due to redundancy and complementarity. Moreover, the fractional objective in \eqref{eq:P1_ratio_obj_bit} couples accumulated task return with accumulated communication cost, and, thus, the system must determine whether each semantic token is worth transmitting for long-horizon closed-loop performance. 
Hence, simple feature-importance scores cannot reliably identify the most useful semantics for long-horizon control performance~\cite{r9,r10,r11}. Third, the edge server has only partial access to the environment through received semantic tokens, while the environment dynamics and the wireless channel are nonlinear, stochastic, and uncertain. It is therefore difficult to directly apply traditional mathematical optimization approaches~\cite{11164706,10943268,10974474,11069265} without strong and often unrealistic analytical assumptions. Here, the use of model-free RL~\cite{10758370,10744542,10915558,10077432} is also not appropriate because reliable estimation of long-horizon return-per-bit from real physical interactions requires extensive trial-and-error exploration, which is data-inefficient and can be unsafe for physical AI agents. These challenges motivate a learning framework that can model closed-loop physical and wireless dynamics, perform counterfactual token value estimation, and jointly optimize token selection and control over extended horizons.

\vspace{-0.6cm}
\section{Proposed World Model-Enabled Causal Digital Twin} 
\vspace{-0.2cm}
Motivated by the above challenges, we develop a WM-CDT framework to solve the return-per-bit problem in \eqref{eq:P1_ratio_obj_bit}. This approach is suitable for this problem because the learned digital twin can capture the coupled evolution of semantic transmission, belief updates, physical control, and environment dynamics, while enabling counterfactual and imagined rollouts without interacting with the real physical system. We first introduce a novel CIV metric for counterfactual valuation of long-horizon information value. We then present the proposed WM-CDT and show how it can learn the closed-loop dynamics of physical AI systems and enable imagined rollouts for CIV estimation and agent control. Next, we present a goal-oriented semantic encoder, an AC control policy, and a CIV-guided token selector. Finally, we analyze the computational complexity and summarize the centralized-training-distributed-execution (CTDE) implementation of the proposed framework.

\subsection{Causal Information Value Metric}
\vspace{-0.1cm}
First, in order to evaluate long-horizon information value of semantic tokens in physical AI systems, the simple approach of directly associating a transmitted token with the observed short-term or single-step return is not suitable because the token-selection decision is history-dependent and affects long-term task performance. Particularly, a token may appear useful simply because it is selected in situations where the agent is already likely to achieve a high return. At the same time, a token that is not useful at the current time step may play a crucial role in the future. Hence, we introduce a novel CIV metric to quantify the marginal causal contribution of each semantic token to long-horizon closed-loop task performance  in physical AI systems. We define $\mathcal{Y}_t$ as the historical states of the wireless physical AI system before token selection at time $t$, which can include the current local observation, past received actions, and the preceding interaction history. The value of a candidate semantic token should be assessed by comparing two hypothetical system evolutions that share the same history $\mathcal{Y}_t$ and differ only in whether that token is transmitted. This requires a \emph{counterfactual comparison} of interventional outcomes. Particularly, we externally set the transmission decision of the token of interest to either $b_{t,n}=1$ or $b_{t,n}=0$, while keeping the remaining semantic context and system dynamics unchanged. 
Based on this principle, we first define the subset-conditioned CIV of any semantic token $\boldsymbol{z}_{t,n}$, $n \in \mathcal{B}_t$, as the change in expected long-horizon return caused by intervening on its transmission as follows.

\vspace{-0.1cm}
\begin{definition}
For a candidate semantic token $\boldsymbol{z}_{t,n}$ within the subset $\mathcal{B}_t$, the \emph{subset-conditioned CIV} is defined as
\begin{equation}
\begin{aligned}
	\label{eq:CIV}
\mathrm{CIV}_{t,n}(\mathcal{B}_t)
&\triangleq 
\mathbb{E}\!\left[
G_t \mid \mathrm{do}(b_{t,n}=1),\,\mathcal{Y}_t,\,\mathcal{B}_t\setminus\{n\}
\right] \\
&- 
\mathbb{E}\!\left[
G_t \mid \mathrm{do}(b_{t,n}=0),\,\mathcal{Y}_t,\,\mathcal{B}_t\setminus\{n\}
\right].
\end{aligned}
\end{equation}
where $\mathrm{do}(b_{t,n}=\cdot)$ is a Pearl's do-operator, which intervenes on the
transmission selection of $\boldsymbol{z}_{t,n}$ at time step $t$, while leaving the remaining token selection unchanged.
\end{definition}

Hence, $\mathrm{CIV}_{t,n}$ measures the marginal long-horizon contribution of the semantic token $\boldsymbol{z}_{t,n}$ by comparing two intervention-consistent system evolutions with the same history and the same remaining semantic context, rather than evaluating single-step reward correlation. CIV provides a token-level causal valuation of how much a semantic token contributes to the long-horizon return term \(R(\pi)\) for the return-per-bit optimization problem in \eqref{eq:P1_ratio_obj_bit}. When combined with the bit cost of the token, CIV can further guide semantic selection toward tokens that provide larger long-horizon task improvement per transmitted bit. In this way, CIV evaluates the semantic utility at the system level, where it captures not only how informative a token is at the current time step, but also how transmitting that token changes future belief evolution, control actions, and ultimately the long-horizon task return. However, it is challenging to evaluate \(\mathrm{CIV}_{t,n}\) in a real wireless physical AI system. In particular, \(\mathrm{CIV}_{t,n}\) is a counterfactual quantity and its evaluation requires comparing two future trajectories that share the same history but differ only in whether token \(n\) is transmitted. Such intervention-consistent comparisons are often unavailable from raw observations alone and cannot be obtained by naive correlation-based estimation.  This motivates the introduction of a world-model-enabled digital twin, which serves as a latent simulator for counterfactual semantic evaluation, long-horizon decision-making, and high data efficiency, as introduced next.

\vspace{-0.3cm}
\subsection{World Model-Enabled Digital Twins}
\label{subsec:bs_world_model}
To solve the return-per-bit maximization problem in \eqref{eq:P1_ratio_obj_bit}, the edge server needs a predictive model that can reason about how transmitted tokens affect future belief updates, control actions, physical state evolution, and long-horizon task return. Direct evaluation in the real system is inefficient and potentially unsafe, since each token selection and control policy may induce a different closed-loop trajectory. Hence, we construct a WM-CDT at the edge server as a learned latent simulator of the closed-loop physical AI system. The digital twin learns a compact latent representation of the environment-agent dynamics from received semantic tokens, executed actions, and rewards. Based on this learned model, the edge server can perform imagined rollouts to estimate long-horizon returns under candidate control policies and counterfactual semantic transmission interventions. Hence, the digital twin provides the foundation for both CIV estimation and long-horizon control planning under partial system observability and wireless bit-budget constraints. In particular, we introduce an RSSM-based digital twin:
\vspace{-0.2cm}
\begin{subequations}
  \label{eq:rssm}
\begin{align}
\text{Deterministic State:}\qquad \boldsymbol{h}_t &= \Psi_{\phi}(\boldsymbol{h}_{t-1}, \boldsymbol{u}_{t-1}, \boldsymbol{a}_{t-1}),\\
\text{Belief Update:}\qquad \boldsymbol{u}_t &\sim q_{\phi}(\boldsymbol{u}_t \mid \boldsymbol{h}_{t}, \boldsymbol{y}_t),\\
\text{Prior Dynamics:}\qquad \tilde{\boldsymbol{u}}_t &\sim p_\phi\!\left(\tilde{\boldsymbol{u}}_t \mid \boldsymbol{h}_t\right), \\
\text{Decoder:}\qquad \hat{\boldsymbol{y}}_t &\sim p_\phi\!\left(\hat{\boldsymbol{y}}_t \mid \boldsymbol{h}_t,\boldsymbol{u}_t\right), \\
\text{Reward Prediction:}\qquad \hat{r}_t &\sim p_\phi\!\left(\hat{r}_t \mid \boldsymbol{h}_t,\boldsymbol{u}_t\right),
\end{align}
\end{subequations}
where $\Psi_{\phi}(\cdot)$ is the deterministic transition model, $p_\phi(\tilde{\boldsymbol{u}}_t \mid \boldsymbol{h}_t)$ is the latent prior model, $q_\phi(\boldsymbol{u}_t \mid \boldsymbol{h}_t,\boldsymbol{y}_t)$ is the posterior inference model that updates the latent belief by using the received semantic tokens, $p_\phi(\boldsymbol{y}_t \mid \boldsymbol{h}_t,\boldsymbol{u}_t)$ is the decoder of semantic tokens, and $p_\phi(\hat{r}_t \mid \boldsymbol{h}_t,\boldsymbol{u}_t)$ is the reward prediction model. 
We adopt the RSSM-based world model in \eqref{eq:rssm} to obtain a probabilistic belief dynamics model for the partially observable closed-loop physical AI system. RSSM can jointly learn posterior belief inference from received semantic tokens, prior dynamics for token-free imagination, and semantic/reward prediction. This structure enables intervention-consistent imagined rollouts under hypothetical semantic transmissions and control actions, thereby supporting both counterfactual CIV estimation and long-horizon planning without additional real-environment interactions.

Given data $\{(\boldsymbol{y}_t,\boldsymbol{a}_t,r_t)\}_{t=0}^{T}$ collected from actual system interactions, the world model can be trained by maximizing the following variational evidence lower bound (ELBO) \cite{r21}:
\begin{align}
	\label{eq:losswm}
\mathcal{L}_{\mathrm{wm}}(\phi)&
=\mathbb{E}_{q_\phi}\!\!\left[\sum_{t=0}^{T}
\underbrace{\log p_\phi(\hat{\boldsymbol{y}}_t \mid \boldsymbol{h}_t,\boldsymbol{u}_t)
+\log p_\phi(\hat{r}_t \mid \boldsymbol{h}_t,\boldsymbol{u}_t)}_{\text{preserve predictive information}}\right. \nonumber \\
&\quad \left.-
\sum_{t=0}^{T}
\underbrace{\mathrm{KL}\!\left(
q_\phi(\boldsymbol{u}_t \mid \boldsymbol{h}_t,\boldsymbol{y}_t)
\,\|\, 
p_\phi(\tilde{\boldsymbol{u}}_t \mid \boldsymbol{h}_t)
\right)}_{\text{posterior belief toward the latent prior}}\right].
\end{align}

\subsection{Temporal Difference (TD)-Based Semantic Encoder}
\label{subsec:selfsup_encoder}
The semantic encoder $f_{\omega}(\cdot)$ is designed to extract compact semantic tokens at the physical agent. However, the conventional reconstruction-based encoder~\cite{a1} usually preserves features for reproducing the original observations, which is not fully aligned with the return-per-bit objective in \eqref{eq:P1_ratio_obj_bit}, since sensory details that are useful for observation reconstruction may not contribute to closed-loop control, but may still occupy semantic tokens and consume wireless resources. Hence, the goal-oriented semantic encoder should extract compact and control-relevant tokens for long-horizon control optimization.
To this end, we first train the encoder with an action-conditioned TD objective, so that the learned tokens capture semantic features related to future system evolution rather than merely preserving observation-level details~\cite{hansen2023td}. Let $d_{\omega}(\cdot)$ be a TD dynamics head. Given the current semantic tokens $\boldsymbol{z}_t=f_{\omega}(\boldsymbol{o}_t)$ and the executed action $\boldsymbol{a}_t$, the next-step semantic tokens are predicted as
$\hat{\boldsymbol{z}}_{t+1}=d_{\omega}(\boldsymbol{z}_t,\boldsymbol{a}_t)$.
To further align the learned representation with long-horizon task utility and avoid representation collapse, we attach a semantic value head $\Omega_{\chi}(\cdot)$ to the encoder during centralized training, which predicts the long-horizon discounted return. The semantic encoder is then trained by minimizing
\vspace{-0.2cm}
\begin{equation}
    \mathcal{L}_{\mathrm{enc}}(\omega,\chi)
    =\left\|\hat{\boldsymbol{z}}_{t+1}-\mathrm{sg}\!\left(\boldsymbol{z}_{t+1}\right)\right\|_2^2
    +\beta_{\chi}\left(\Omega_{\chi}(\boldsymbol{z}_t)-G_t\right)^2,
\end{equation}
where $\mathrm{sg}(\cdot)$ is the stop-gradient operator, and $\beta_{\chi}\in(0,1)$ is the weighting factor.

\vspace{-0.3cm}
\subsection{Actor-Critic-Based Agent Control}
\label{subsec:ac_planning}
Given the WM-CDT, we optimize the control policy of the physical agent, which must account for the coupled effects of sensing, semantic transmission, inference, and physical dynamics over an extended horizon. Hence, we adopt an AC-based RL approach as the planner of the world model for agent control to solve the long-horizon return-per-bit problem in \eqref{eq:P1_ratio_obj_bit}. We approximate the partially observable problem by a partially observable Markov decision process (POMDP) with the state space $\mathcal{S}$ and the  action space $\mathcal{A}$, where the state is $\{\boldsymbol{h}_t,\boldsymbol{u}_t\}$, and the action is the agent control $\boldsymbol{a}_t$.

However, it is inefficient to learn the policy directly from the fractional return-per-bit objective in \eqref{eq:P1_ratio_obj_bit}.  Since the objective increases when the communication cost is reduced, early-stage training may bias the selector toward transmitting too few semantic tokens, even when these tokens are still needed for reliable belief updates and control learning. To obtain an additive objective suitable for AC learning, we use a Dinkelbach-inspired transformation of the fractional objective. For a given per-bit price $\alpha$, we define the shaped reward as
\vspace{-0.2cm}
\begin{equation}
\tilde r_t \triangleq \hat{r}_t - \alpha\,\ell(\boldsymbol{c}_t),
\label{eq:sr}
\end{equation}
where $\hat r_t$ is the reward prediction by the RSSM model at time slot $t$. The shaped reward penalizes communication cost while preserving the long-horizon task objective, thereby aligning both control learning and token selection with return-per-bit maximization. Theorem~\ref{thm:dinkelbach_main} characterizes the optimal return-per-bit policy through priced additive subproblems as follows.

\begin{theorem}
\label{thm:dinkelbach_main}
Let \(\Pi\) be a nonempty feasible policy set. Assume that the discounted
communication cost is strictly positive for every feasible policy, i.e.,
\(C(\pi)>0\) for all \(\pi\in\Pi\). Assume that the suprema are attained over \(\Pi\).
Define the return-per-bit objective $\eta(\pi)\triangleq R(\pi)/C(\pi)$ and its optimal value $\eta^*\triangleq \max_{\pi\in\Pi}\eta(\pi)$.
For any $\alpha\in\mathbb{R}$, define
\vspace{-0.15cm}
\begin{equation}
F(\alpha)\triangleq \sup_{\pi\in\Pi}\big(R(\pi)-\alpha C(\pi)\big).
\end{equation}
Then $\eta^*$ is the unique scalar satisfying $F(\eta^*)=0$, and any ratio-optimal policy $\pi^*$ maximizes the priced objective at $\alpha=\eta^*$ and achieves zero gap:
\vspace{-0.2cm}
\begin{equation}
\pi^* \! \in \! \arg\max_{\pi\in\Pi}\big(R(\pi)-\eta^* C(\pi)\big),
\,
R(\pi^*)-\eta^* C(\pi^*)=0.
\end{equation}
\end{theorem}
\vspace{-0.3cm}
\begin{proof}
See the Appendix.
\end{proof}
\vspace{-0.1cm}

Theorem~\ref{thm:dinkelbach_main} shows that the fractional return-per-bit objective can be represented through an additive priced objective at the optimal communication price $\alpha=\eta^*$. Hence, optimizing the cumulative shaped reward $\hat{r}_t-\alpha\ell(\boldsymbol{c}_t)$ in \eqref{eq:sr} allows AC learning to trade off task return and communication cost while remaining consistent with the original return-per-bit objective.

Based on the shaped reward $\tilde r_t$, we optimize the control policy with AC learning in the latent space of the causal digital twin, where the actor model is represented by $\boldsymbol{a}_t \sim \pi^2_{\theta}(\cdot \mid \boldsymbol{h}_t, \boldsymbol{u}_t)$ and the critic model with the parameter $\psi$ estimates the discounted return:
\vspace{-0.2cm}
\begin{equation}
V_{\psi}(\boldsymbol{h}_t,\boldsymbol{u}_t)
\approx
\mathbb{E}_{\pi}\!\left[
\sum_{\tau=t}^{T}\gamma^{\tau-t}\tilde r_{\tau}
\,\bigg|\,
\boldsymbol{h}_t,\boldsymbol{u}_t
\right].
\end{equation}
The actor model and the critic model are trained over the imagined rollouts generated by the learned world model. In particular, at each imagined step $\tau$, an action is sampled from $\pi^2_{\theta}(\cdot\mid \boldsymbol{h}_{\tau},\tilde{\boldsymbol{u}}_{\tau})$, the RSSM dynamics propagates the latent belief to $(\boldsymbol{h}_{\tau+1},\tilde{\boldsymbol{u}}_{\tau+1})$, and the reward model predicts the corresponding shaped reward $\tilde r_{\tau}$. Then, an \emph{imagined trajectory} $\tilde{\mathcal{J}} = \{(\boldsymbol{h}_{\tau},\tilde{\boldsymbol{u}}_{\tau},\boldsymbol{a}_{\tau},\tilde r_{\tau})\}_{\tau=t}^{t+H-1}$ with horizon size $H$ is obtained, which is used for control learning without requiring costly trial-and-error interaction in the real system.
To stabilize long-horizon value estimation, we adopt a $\lambda$-return target from the imagined rollout. The $k$-step bootstrapped return is
\vspace{-0.2cm}
\begin{equation}
G^{(k)}_{t}
\triangleq
\sum_{j=0}^{k-1}\gamma^{j}\tilde r_{t+j}
+
\gamma^{k}V_{\psi}(\boldsymbol{h}_{t+k},\tilde{\boldsymbol{u}}_{t+k}), k=1,\ldots,H,
\end{equation}
and the corresponding $\lambda$-return target is
\vspace{-0.2cm}
\begin{equation}
V_{\lambda}(\boldsymbol{h}_t,\tilde{\boldsymbol{u}}_t)
\triangleq
(1-\lambda)\sum_{k=1}^{H-1}\lambda^{k-1}G^{(k)}_{t}
+
\lambda^{H-1}G^{(H)}_{t},
\end{equation}
where $\lambda\in[0,1]$. The critic is updated by minimizing
\vspace{-0.2cm}
\begin{equation}
\mathcal{L}_{\mathrm{critic}}(\psi)
=
\mathbb{E}\left[
\big(
V_{\psi}(\boldsymbol{h}_t,\tilde{\boldsymbol{u}}_t)
-
V_{\lambda}(\boldsymbol{h}_t,\tilde{\boldsymbol{u}}_t)
\big)^2
\right].
\end{equation}
The actor is trained to maximize the predicted long-horizon return along imagination:
\vspace{-0.2cm}
\begin{equation}
\begin{aligned}
\mathcal{L}_{\mathrm{actor}}(\theta)
&=
-\mathbb{E}\left[
\sum_{\tau=t}^{t+H-1}
\gamma^{\tau-t}
V_{\lambda}(\boldsymbol{h}_{\tau},\tilde{\boldsymbol{u}}_{\tau})
\right] \\
&\qquad \qquad \qquad
-\beta
\mathbb{E}\!\left[
\mathcal{H}\!\left(
\pi_{2}^{\theta}(\cdot \mid \boldsymbol{h}_{\tau},\tilde{\boldsymbol{u}}_{\tau})
\right)
\right],
\end{aligned}
\end{equation}
where $\mathcal{H}(\cdot)$ represents the policy entropy and $\beta\ge 0$ is the exploration regularization. In this way, the edge server learns a control policy that is optimized not only for immediate task reward, but also for its long-horizon impact on future closed-loop performance under communication cost.

\subsection{Semantic Token Selection by Counterfactual Reasoning}
\label{subsec:wm_civ_uav_select}
We next optimize the semantic token selection policy $\pi^1_{\varpi}(\cdot \mid \boldsymbol{z}_{t})$ at the physical agent. At time $t$, the agent extracts a set of candidate semantic tokens $\boldsymbol{z}_t \triangleq \{\boldsymbol{z}_{t,1},\ldots,\boldsymbol{z}_{t,N}\}$,
where token $n$ has code length $\ell_{t,n}$ bits. Due to bit budget $U_t$, the selected token subset must satisfy the per-slot uplink bit budget $\sum_{n=1}^{N} b_{t,n}\ell_{t,n}\le U_t$.
The key challenge is that semantic tokens are not independent under shared resource constraints. Their utility depends on the other retained tokens due to redundancy, complementarity, and their coupled impact on future belief evolution and action decisions. To address this issue, we develop a gate network-based selection policy with counterfactual reverse pruning as follows.

\subsubsection{Estimated CIV}
Given a candidate subset  $\mathcal{B}_t \subseteq \{1,\ldots,N\}$, the subset-conditioned CIV of token $n$ defined in \eqref{eq:CIV} measures the marginal long-horizon contribution of a semantic  token under the same semantic context. By using the causal digital twin, the CIV can be approximated by using imagined rollouts. In particular, we define
\vspace{-0.15cm}
\begin{equation}
\hat{Q}_t(\mathcal{B}_t) \triangleq 
\mathbb{E}\left[
\sum_{\tau=t}^{t+H-1} \gamma^{\tau-t} \hat{r}_{\tau}
\,\middle|\, \mathcal{B}_t
\right]
\end{equation}
as the digital-twin-predicted long-horizon return after updating the 
belief with the token subset $\mathcal{B}_t$ and rolling out future actions by 
using $\pi_{\theta}^{2}$. Then, the subset-conditioned CIV can be estimated as
\vspace{-0.15cm}
\begin{equation}
\widehat{\mathrm{CIV}}_{t,n}(\mathcal{B}_t)
\triangleq
\hat{Q}_t(\mathcal{B}_t)
-
\hat{Q}_t(\mathcal{B}_t \setminus \{n\}).
\end{equation}
To account for communication efficiency, we further define the estimated 
CIV-per-bit of token $n$ under subset $\mathcal{B}_t$ as
\vspace{-0.15cm}
\begin{equation}
\hat{\rho}_{t,n}(\mathcal{B}_t)
\triangleq
\frac{
\widehat{\mathrm{CIV}}_{t,n}(\mathcal{B}_t)
}{
\ell_{t,n}+\epsilon_{\rho}
},
\end{equation}
where $\epsilon_{\rho} > 0$ is a small constant for numerical stability. While 
$\widehat{\mathrm{CIV}}_{t,n}(\mathcal{B}_t)$ measures absolute long-horizon 
utility, $\hat{\rho}_{t,n}(\mathcal{B}_t)$ measures how efficiently that utility 
is delivered per bit. Similar to Theorem 1, under the Dinkelbach price $\alpha$, 
we define the priced marginal contribution:
\vspace{-0.2cm}
\begin{equation}
\hat{\Gamma}_{t,n}(\mathcal{B}_t)
\triangleq
\widehat{\mathrm{CIV}}_{t,n}(\mathcal{B}_t)
-
\alpha \ell_{t,n},
\label{eq:pciv}
\end{equation}
which characterizes whether retaining token $n$ is beneficial after taking into account its communication cost.

\subsubsection{Gate prediction and causal reverse pruning}
The selector $\pi^1_{\varpi}(\cdot \mid \boldsymbol{z}_{t})$ is parameterized by a gate network that predicts a retention score and probability for each candidate token:
\vspace{-0.2cm}
\begin{equation}
s_{t,n}=g_{\varpi}(\boldsymbol{z}_{t,n}),
\qquad
p_{t,n}=\sigma(s_{t,n}),
\label{eq:gate_prob}
\end{equation}
where $g_{\varpi}(\cdot)$ is the gate model and $\sigma(\cdot)$ is the sigmoid function. During the training stage, instead of using these probabilities as final independent transmission decisions, we use them to construct an over-complete initial proposal set:
\vspace{-0.2cm}
\begin{equation}
\mathcal{B}_{t}^{(0)}
\triangleq
\mathrm{TopM}\!\left(\{p_{t,n}\}_{n=1}^{N};M_t\right),
\label{eq:init_set}
\end{equation}
where $\mathrm{TopM}(\{p_{t,n}\}_{n=1}^{N};M_t)$ is the operator that returns the indices of the $M_t$ tokens with the largest gate probabilities among $\{p_{t,n}\}_{n=1}^{N}$, and $M_t$ is the maximum number of semantic tokens that satisfies the bit budget $U_t$. Starting from $\mathcal{B}_{t}^{(0)}$, the final transmitted subset is obtained by counterfactual reverse pruning. In particular, let $\mathcal{B}_{t}^{(m)}$ be the subset at pruning step $m$. For each token $n\in\mathcal{B}_{t}^{(m)}$, the digital twin computes its contextual contribution through \eqref{eq:pciv}. Then, we remove the token with the lowest priced marginal contribution:
\vspace{-0.2cm}
\begin{equation}
n_{t,\mathrm{rm}}^{(m)}
=
\arg\min_{n\in\mathcal{B}_{t}^{(m)}}
\hat{\mathrm{\Gamma}}_{t,n}\!\left(\mathcal{B}_{t}^{(m)}\right),
\label{eq:prune_by_priced_civ}
\end{equation}
and update
$
\mathcal{B}_{t}^{(m+1)}=\mathcal{B}_{t}^{(m)}\setminus\{n_{t,\mathrm{rm}}^{(m)}\}.
$
This pruning procedure is repeated until the bit-budget constraint is satisfied and all retained tokens have a positive priced marginal contribution. The final selected subset is denoted by $\mathcal{B}_{t}^{*}$, and the hard selection mask is thus given by $b_{t,n}^{*}=\mathbb{I}\!\left(n\in\mathcal{B}_{t}^{*}\right)$. The above reverse-pruning strategy repeatedly deletes the token with the smallest contextualized long-horizon contribution, while CIV-per-bit and the priced contribution further improve communication efficiency.

To align the gate network with the reverse-pruning solution, we use the hard pruning decisions as self-supervised targets for CTDE, and, thus, the final selector loss is
\vspace{-0.2cm}
\begin{equation}
\mathcal{L}_{\mathrm{sel}}(\varpi)
\!=\!-\!\sum_{n=1}^{N}\!\left[
b_{t,n}^{*}\log p_{t,n}
\!+\!
(1\!-\!b_{t,n}^{*})\log(1\!-\!p_{t,n})
\right],
\label{eq:sel_total}
\end{equation}
where the gate network $\pi^1_{\varpi}$ first assigns an initial retention probability $p_{t,n}$ to each semantic token and forms a high-recall candidate set. The causal digital twin then refines this candidate set through subset-conditioned CIV evaluation, CIV-per-bit-aware reverse pruning, and priced contribution testing.

Inspired by Dinkelbach's method, at each policy-update stage, the adaptive price is updated by the clipped empirical ratio between the discounted return and communication cost as
$
\alpha\leftarrow 
\max\left\{0,\frac{\hat R(\pi)}{\hat C(\pi)+\epsilon}\right\},
$
where $\epsilon>0$ is a small constant for numerical stability, and $\hat R(\pi)$ and $\hat C(\pi)$ are computed from the latest batch of trajectories. The clipping operation prevents unstable pricing when the empirical discounted return is negative, it is an implementation heuristic. 

\vspace{-0.4cm}
\subsection{Optimality, Complexity, and Implementation}
The proposed framework learns a feasible suboptimal policy that improves the empirical return-per-bit while satisfying the bit-budget constraint, thus providing an approximate solution to the return-per-bit maximization problem in \eqref{eq:P1_ratio_obj_bit}. 
Let $\mathcal{C}_{\rm net}$ be the average computational cost of one neural-network forward/backward step for the world model, selector, and AC modules. 
For each training step, world-model learning over a sampled sequence of length $L$ requires $\mathcal{O}(L\mathcal{C}_{\rm net})$, AC learning over imagination of horizon $H$ requires $\mathcal{O}(H\mathcal{C}_{\rm net})$, and CIV-guided reverse pruning over $N$ candidate tokens requires $\mathcal{O}(N^2H\mathcal{C}_{\rm net})$ because counterfactual rollouts are evaluated for different token-removal decisions. 
Thus, the total training complexity over $I$ iterations is $\mathcal{O}(I(L+H+N^2H)\mathcal{C}_{\rm net})$. Let $\mathcal{C}_{\rm inf}$ be the one-step forward inference cost per time slot. 
The online complexity is $\mathcal{O}(\mathcal{C}_{\rm inf}+N\log N)$ per slot, which has near-linear complexity and can support real-time online inference for physical AI systems.

The proposed WM-CDT can be deployed in an edge-assisted cellular architecture, where the physical AI agent operates as a wireless user equipment (UE) and the digital twin is hosted at an edge server co-located with the serving base station. 
The physical agent needs to only execute lightweight on-device modules, including sensing, semantic encoding, token selection, and transmission, whereas computationally intensive modules, such as RSSM training, CIV-based counterfactual evaluation, and actor-critic policy optimization, are performed at the edge server. 
This separation follows a CTDE principle. During centralized training, the edge server builds the replay buffer $\mathcal{D}$ and requires access to the full semantic token set $\boldsymbol{z}_t$ at each training slot. 
Hence, the training-stage uplink overhead scales with the full candidate-token description, i.e., $\sum_{n=1}^{N}\ell_{t,n}$ bits per slot, together with low-dimensional records such as the executed action $\boldsymbol{a}_t$, reward $r_t$, and wireless reception outcome. 
For the simulation setting with $N=32$ candidate tokens and $8$ bits per token, this corresponds to $256$ bits transmission cost per slot. 
This overhead is incurred only during centralized training. 

During online execution, at each slot, given available bit budget $U_t$, the local selector $\pi_{\varpi}^{1}$ chooses a subset $\mathcal{B}_t$ of semantic tokens whose encoded packet length satisfies $\ell(c_t)\le U_t$. 
The selected tokens are then quantized by a shared semantic codebook~\cite{11366555}, channel-coded, and transmitted over the cellular network's uplink to the edge server. 
After physical-layer decoding, the base station forwards the recovered semantic tokens $\boldsymbol{y}_t$ to the edge server-hosted digital twin, which updates the latent belief using $\Psi_{\phi}$, $q_{\phi}$, and $p_{\phi}$, computes the control action using $\pi_{\theta}^{2}$, and returns only a low-dimensional control command through the downlink control/data channel. 
Thus, the online system requires no real-time counterfactual rollout on the agent and introduces only one semantic-packet uplink transmission and one compact control downlink feedback per closed-loop decision slot. 
The overall training procedure is summarized in Algorithm~\ref{alg:overall_training}.

\setlength{\textfloatsep}{1pt} 
\setlength{\intextsep}{1pt} 
\begin{algorithm}[t!]
\caption{Training of the Proposed WM-CDT Framework}
\label{alg:overall_training}
\scriptsize
\begin{algorithmic}[1]
\REQUIRE Initialize model parameters $\omega$, $\varpi$, $\phi$, $\theta$, $\psi$, and $\alpha$

\item[] \textcolor{gray}{// Stage 1: Self-supervised semantic encoder pretraining}
\FOR{each local transition $(\boldsymbol{o}_t,\boldsymbol{a}_t,\boldsymbol{o}_{t+1})$}
    \STATE Update $\omega$ by minimizing the TD loss
\ENDFOR

\item[] \textcolor{gray}{// Stage 2: Training of digital twin, control, and selector}

\FOR{training iteration $i=1,2,\ldots,I$}

\item[] \textcolor{gray}{// Real Environment Interaction \& Data Collection}
    \FOR{each environment step $t$}
        \STATE Observe local sensing data $\boldsymbol{o}_t$ at the physical agent
        \STATE Extract candidate semantic tokens using $f_{\omega}(\cdot)$
        \STATE Form a transmitted token subset $\mathcal{B}_t$ by using $\pi_{\varpi}^{1}$
        \STATE Encode the selected tokens into message $c_t$ and transmit it over the wireless uplink
        \STATE Decode semantic tokens $\boldsymbol{y}_t$ at the decision maker
        \STATE Update the latent belief by using $\Psi_{\phi}$, $q_{\phi}$ and $p_{\phi}$
        \STATE Select the control action by using $\pi_{\theta}^{2}$
        \STATE Feed the control action back to the physical agent
        \STATE Store $\left(\boldsymbol{o}_t,\boldsymbol{y}_t,\boldsymbol{a}_t,r_t\right)$ to the replay buffer $\mathcal{D}$
    \ENDFOR

	\STATE Sample sequences $\{(\boldsymbol{o}_t,\boldsymbol{y}_t,\boldsymbol{a}_t,r_t)\}_{t=k}^{k+L} \sim \mathcal{D}$.
  \STATE Update $\omega$ and $\chi$ by minimizing $\mathcal{L}_{\mathrm{enc}}(\omega,\chi)$
  \STATE Update the world model $\phi$ by maximizing $\mathcal{L}_{\mathrm{wm}}(\phi)$
	\item[] \textcolor{gray}{// Training of Control and Selection by imagination}
	\STATE Compute the optimal subset $\mathcal{B}_{t}^{*}$ by counterfactual reasoning from $z_t=f_{\omega}(\boldsymbol{o}_t)$ (Imagination over $\{p_{t,n}\}$)
	\STATE Update the selector by minimizing $\mathcal{L}_{\mathrm{sel}}(\varpi)$
	\STATE Imagine $\{(\tilde{\boldsymbol{u}}_\tau, a_\tau)\}_{\tau=t}^{t+H}$ over actions from the AC
    \STATE Compute $\lambda$-return targets, update critic $\psi$ by using $\mathcal{L}_{\mathrm{critic}}$ and actor $\theta$ by using $\mathcal{L}_{\mathrm{actor}}$, and update $\alpha$
\ENDFOR
\end{algorithmic}
\end{algorithm}

\begin{figure}[t!]
  \centering
  \includegraphics[width=0.9
  \linewidth]{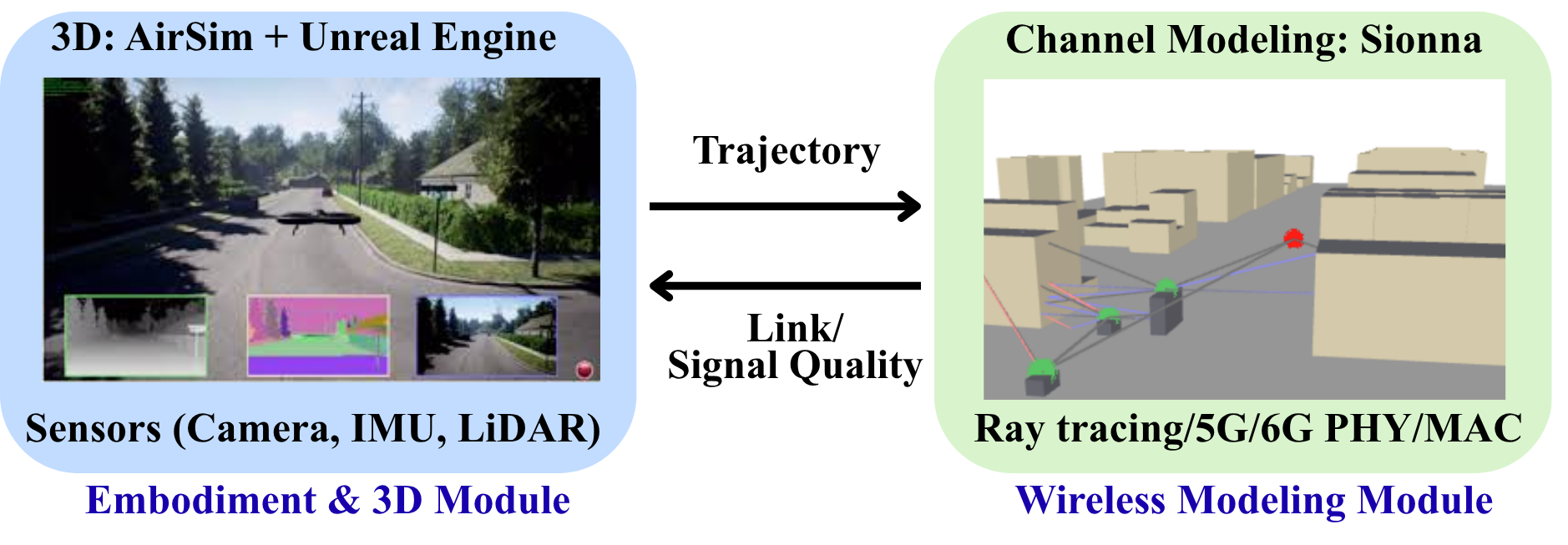}
  \vspace{-0.4cm}
  \caption{The proposed AirSim-Sionna-based simulator for semantic communications in closed-loop physical AI systems.}
  \label{fig:sim}
\end{figure}

\vspace{-0.3cm}
\section{Simulation Results and Analysis}
\vspace{-0.1cm}
\subsection{AirSim-Sionna-Based Realistic Simulator}
\vspace{-0.1cm}
To evaluate semantic communications in closed-loop physical AI systems under realistic interactions, we develop a novel AirSim-Sionna-based simulator that jointly models physical agent dynamics, visual sensing, semantic transmission, wireless channel effects, remote inference, and control, as shown in Fig. 2. The proposed simulator is designed to reproduce the sensing-communication-inference-control loop considered in our system whereby a physical agent interacts with a three-dimensional environment, extracts task-relevant semantic tokens from local observations, transmits a selected subset of tokens over a wireless uplink, and receives control commands generated by an edge server. 

The simulator consists of two coupled modules. The first module is the physical environment module, implemented in AirSim \cite{shah2017airsim}. AirSim provides a high-fidelity three-dimensional environment in which the physical agent moves according to realistic kinematics. The second module is the Sionna-based wireless communication module \cite{sionna}. Particularly, the Sionna-based channel module simulates the uplink transmission from the physical agent to the edge server. In particular, the channel state is generated according to the relative geometry between the mobile agent and the receiver, together with propagation effects such as path loss, fading, multipath propagation, and additive noise.  The AirSim and Sionna modules are synchronized at the time-slot level. Each simulation step corresponds to one closed-loop interaction round. Particularly, the simulator first builds the physical environment in AirSim, where the agent collects the local observation, extracts candidate semantic tokens, and transmits selected tokens. The selected message is then passed through the Sionna wireless channel, decoded at the edge server, and used to update the digital twin. Finally, the control policy generates the control command, which is applied to the AirSim agent in the next physical step.

\vspace{-0.4cm}
\subsection{Parameter and System Settings}
We consider a wireless semantic UAV navigation task in the AirSim-Sionna-based simulator. The UAV flies in a $200\,\mathrm{m}\times 200\,\mathrm{m}$ three-dimensional environment with randomly placed obstacles. For each episode, the UAV starts from the center of the environment, and the destination is randomly generated on a circle with radius $70\,\mathrm{m}$. An episode is successful if the UAV reaches the destination when it is within a $5\,\mathrm{m}$ radius of the target location, and it fails if a collision with obstacles occurs. If neither success nor collision occurs, the episode is truncated at the maximum horizon $T=200$.
At each time slot, the UAV observes an RGB image and extracts $N=32$ candidate semantic tokens. Each semantic token is represented by a $64$-dimensional continuous embedding, which is quantized to $8$ bits by using a codebook-based vector quantizer before transmission~\cite{11366555}. Hence, the communication cost of the selected semantic message is
$\ell(c_t)=8|\mathcal{B}_t|$, where $\mathcal{B}_t$ is the selected token subset. The selected tokens must satisfy the per-slot bit-budget constraint $\ell(c_t)\leq U_t$, where $U_t\in\{64,96,128,160,192\}$ bits/slot. The convergence experiment uses $U_t=96$ bits/slot. The remote server generates the continuous UAV control command
$\boldsymbol{a}_t=[v_t^{\mathrm{f}}, v_t^{\mathrm{z}}, \omega_t]$, where
$v_t^{\mathrm{f}}$ is the forward velocity, $v_t^{\mathrm{z}}$ is the vertical velocity, and $\omega_t$ is the steering angular velocity. The task reward is defined as
$
    r_t = d_{t-1}-d_t + R_{\rm succ}\mathbb{I}_{\rm succ}
    - R_{\rm col}\mathbb{I}_{\rm col} - \beta_r,
$
where $d_t$ is the distance between the UAV and the target, and $\mathbb{I}_{\rm succ}$ and $\mathbb{I}_{\rm col}$ indicate successful arrival and collision, respectively. We set $R_{\rm succ}=10$, $R_{\rm col}=10$, and $\beta_r=0.01$.
All training hyperparameters are summarized in Table \ref{tab:system_params}.

We compare the proposed WM-CDT scheme with four baselines. 
First, AC-RRL \cite{gc} is a model-free recurrent AC baseline that uses the same TD-pretrained semantic encoder and budget-aware semantic token selector as WM-CDT, but it neither learns a world model nor performs counterfactual reasoning. Instead, it directly optimizes semantic token selection and UAV control from real AirSim-Sionna interactions using the Dinkelbach-shaped reward.
Second, MBPO is a model-based RL baseline that learns a latent dynamics model and uses short imagined rollouts for policy optimization. 
Third, AC \cite{a1} is a feedforward AC baseline. It generates control actions based only on the currently received semantic tokens, without accumulating historical semantic information.
Fourth, PPO is a model-free RL baseline that updates the semantic selection and control policies using the clipped PPO objective under the same semantic bit-budget constraint.

\begin{table}[t]
\centering
\caption{Simulation parameters and learning hyperparameters.}
\label{tab:system_params}
\begin{tabular}{lc}
\toprule
\textbf{Variable} & \textbf{Value / Setting} \\
\midrule
Environment size & $200\,\mathrm{m}\times 200\,\mathrm{m}$ \\
Success radius & $5\,\mathrm{m}$ \\
Maximum horizon $T$ & 200 \\
Control action & $[v_t^{\mathrm{f}}, v_t^{\mathrm{z}}, \omega_t]$ \\
Number of semantic tokens & $N=32$ \\
Token embedding dimension & $64$ \\
Quantization bits per token & $8$ bits \\
Bit budget $U_t$ (bits/slot) & $\{64,96,128,160,192\}$ \\
Success reward $R_{\rm succ}$ & $10$ \\
Collision penalty $R_{\rm col}$ & $10$ \\
Step penalty $\beta_r$ & $0.01$ \\
Discount factor $\gamma$ & $0.99$ \\
GAE parameter $\lambda$ & $0.95$ \\
Trajectory length $L$ & $32$ \\
Hidden layer dimension & $128$ \\
Optimizer & Adam \\
Learning rate & $3\times10^{-4}$ \\
Batch size & $128$ \\
Replay buffer size & $5\times10^4$ \\
Encoder pretraining samples & $5\times10^4$ \\
PPO clipping coefficient & $0.2$ \\
Rollout horizon & $10$ \\
Training steps & $1\times10^6$ \\
Random seeds & $10$ \\
\bottomrule
\end{tabular}%
\end{table}

\vspace{-0.5cm}
\subsection{Convergence Performance}
\begin{figure}[t!]
  \centering
  \includegraphics[width=0.9\linewidth]{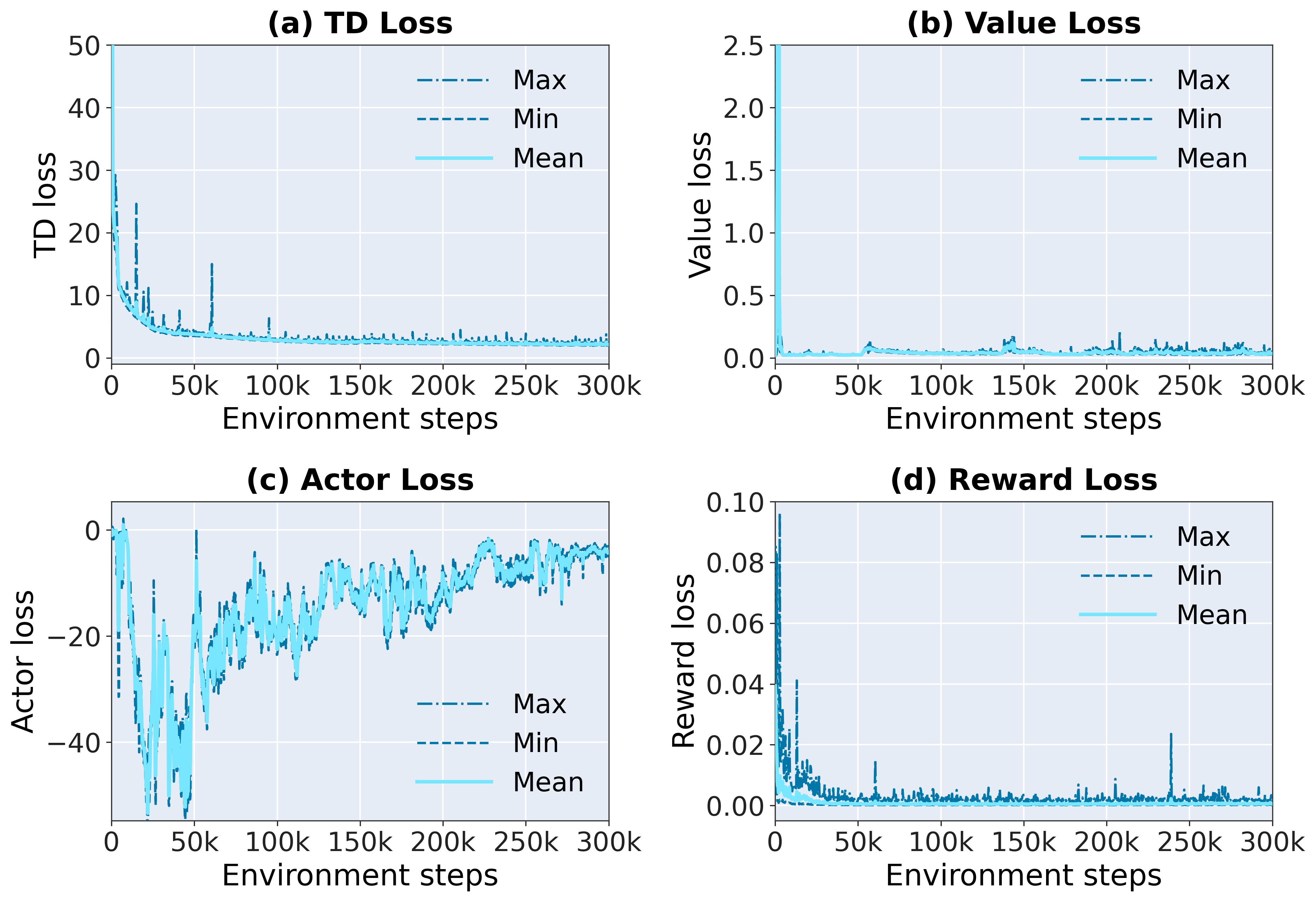}
\vspace{-0.4cm}
  \caption{The training losses of different components in WM-CDT.}
  \label{fig:conv2}
\end{figure}

\begin{figure}[t!]
  \centering
    \vspace{-0.3cm}
  \includegraphics[width=0.8\linewidth]{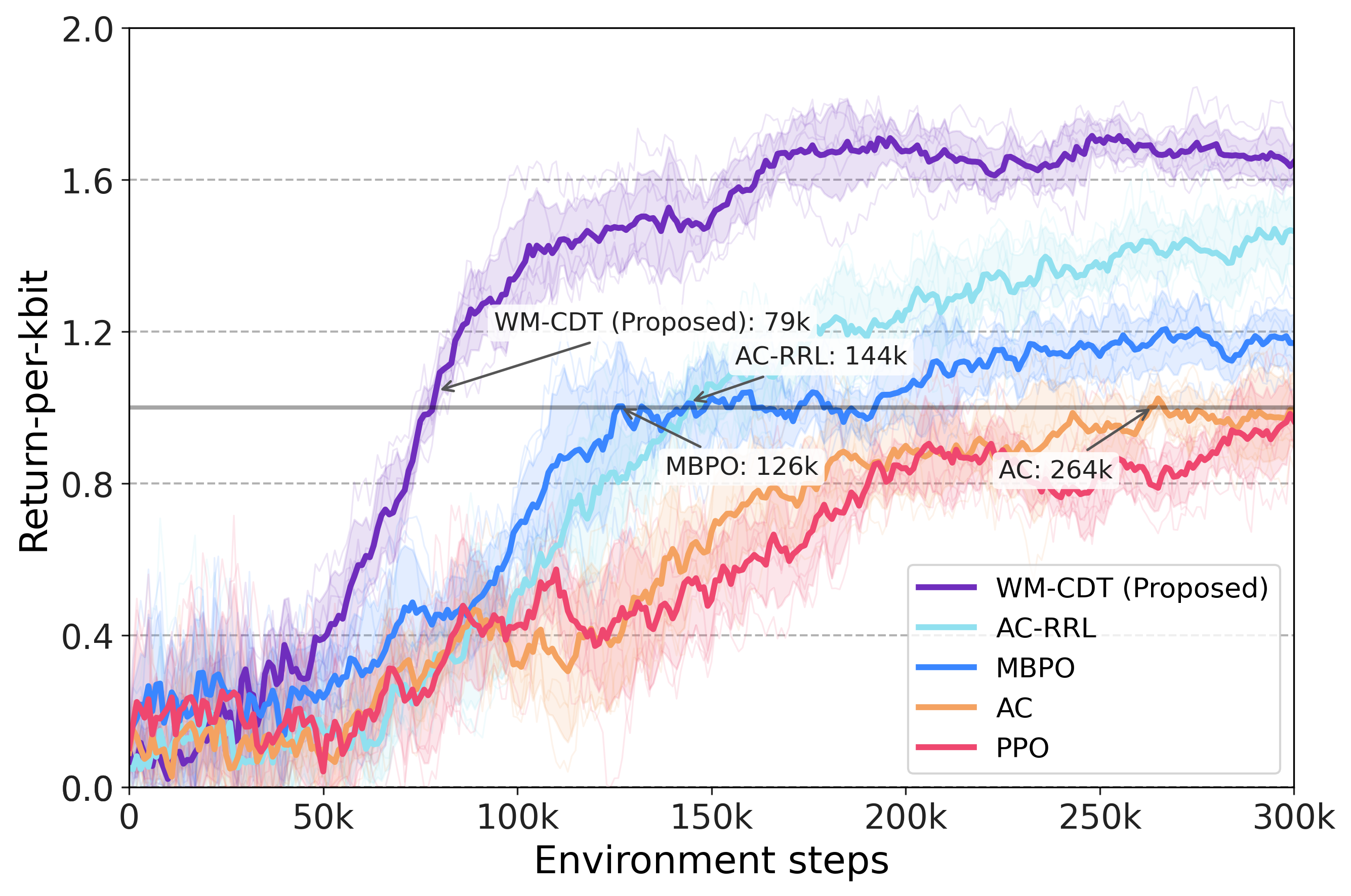}
  \vspace{-0.5cm}
  \caption{The convergence of different methods with limited environment steps.}
  \label{fig:conv}
\end{figure}

Fig.~\ref{fig:conv2} shows the training losses of the proposed WM-CDT approach. From Fig.~\ref{fig:conv2}, we can observe that the TD loss decreases rapidly at the beginning and then gradually converges, which indicates that the semantic encoder can learn temporally predictive token representations from UAV observations. The value loss and reward loss remain small after the initial training stage, which shows that the learned latent dynamics and reward prediction become stable. Meanwhile, the actor loss first fluctuates due to exploration and then becomes more stable as the policy improves. Fig.~\ref{fig:conv2} shows that the world model can be jointly trained toward convergence in a stable manner under limited interactions. 

Fig.~\ref{fig:conv} compares the convergence to a final return-per-kbit for different approaches under limited environment steps. From Fig.~\ref{fig:conv}, the proposed WM-CDT scheme reaches the $1$ return-per-kbit level at $79$k environment steps, while MBPO, AC-RRL, and AC reach the same level at $126$k, $144$k, and $264$k steps, respectively. Hence, WM-CDT achieves significant improvements of $37.3\%$, $45.1\%$, and $70.1\%$ in terms of data efficiency compared to MBPO, AC-RRL, and AC, respectively. Moreover, WM-CDT achieves the highest return-per-kbit among all methods in terms of final performance, WM-CDT achieves about $13.8\%$ higher return-per-kbit than AC-RRL, $39.8\%$ higher than MBPO, $66.7\%$ higher than AC, and $71.9\%$ higher than PPO. This improvement comes from two key components. First, the RSSM-based world model enables long-horizon imagined rollouts, which improves planning efficiency compared with purely model-free methods. Second, the counterfactual token valuation mechanism allows WM-CDT to select semantic tokens according to their long-horizon causal contribution rather than their instantaneous relevance.

\vspace{-0.3cm}
\subsection{Computational Complexity}
Table~\ref{tab:complexity} shows the computational overhead and online latency of different approaches. 
The proposed WM-CDT has the largest model size and training time because it jointly maintains the semantic encoder, RSSM-based world model, AC controller, and counterfactual semantic token evaluation module.  As shown in Fig.~\ref{fig:conv}, this additional training overhead leads to return-per-kbit gains and high data efficiency. We note that the training overhead is only incurred offline at the edge server during centralized training. For decentralized online execution, the latency of WM-CDT remains close to that of the model-free baselines because all methods make a single-step decision at each control slot. In particular, after training, the learned encoder and selector are directly deployed at the physical agent, which only needs one forward pass to extract semantic tokens and select the transmitted subset. At the edge server, WM-CDT performs one RSSM belief update using the received semantic tokens and then one actor forward pass to generate the control action. 
Hence, although WM-CDT introduces additional computation during training, its online latency remains close to that of the model-free baselines. 
WM-CDT achieves an online latency of $6.4$ ms/slot, which is only slightly higher than MBPO, AC-RRL, AC, and PPO, whose online latencies are $6.2$, $5.8$, $5.1$, and $5.3$ ms/slot, respectively. 
The small additional latency of WM-CDT mainly comes from the RSSM belief update before the actor generates the control action. 
Overall, Table~\ref{tab:complexity} and Figs.~\ref{fig:conv2}-\ref{fig:conv} show that WM-CDT introduces moderate computational overhead in exchange for improved convergence, higher return-per-kbit, and better long-horizon semantic control.

\begin{table}[t]
\centering
\caption{Computational overhead and inference latency resulting from WM-CDT and Baselines.}
\label{tab:complexity}
\begin{tabular}{lccc}
\toprule
\textbf{Method} 
& \makecell{\textbf{Parameters}\\\textbf{(M)}} 
& \makecell{\textbf{Training Time}\\\textbf{(ms/step)}} 
& \makecell{\textbf{Online Latency}\\\textbf{(ms/slot)}} 
\\
\midrule
WM-CDT   & 8.6  & 28.4 & 6.4 \\
AC-RRL   & 5.9  & 16.8 & 5.8 \\
MBPO     & 7.3  & 27.6 & 6.2 \\
AC       & 5.4  & 14.9 & 5.1 \\
PPO      & 6.6  & 15.7 & 5.3 \\
\bottomrule
\end{tabular}
\end{table}

\vspace{-0.5cm}
\subsection{Task Performance}
\vspace{-0.1cm}

\begin{figure}[t!]
  \centering
 \vspace{-0.3cm}
  \includegraphics[width=0.8\linewidth]{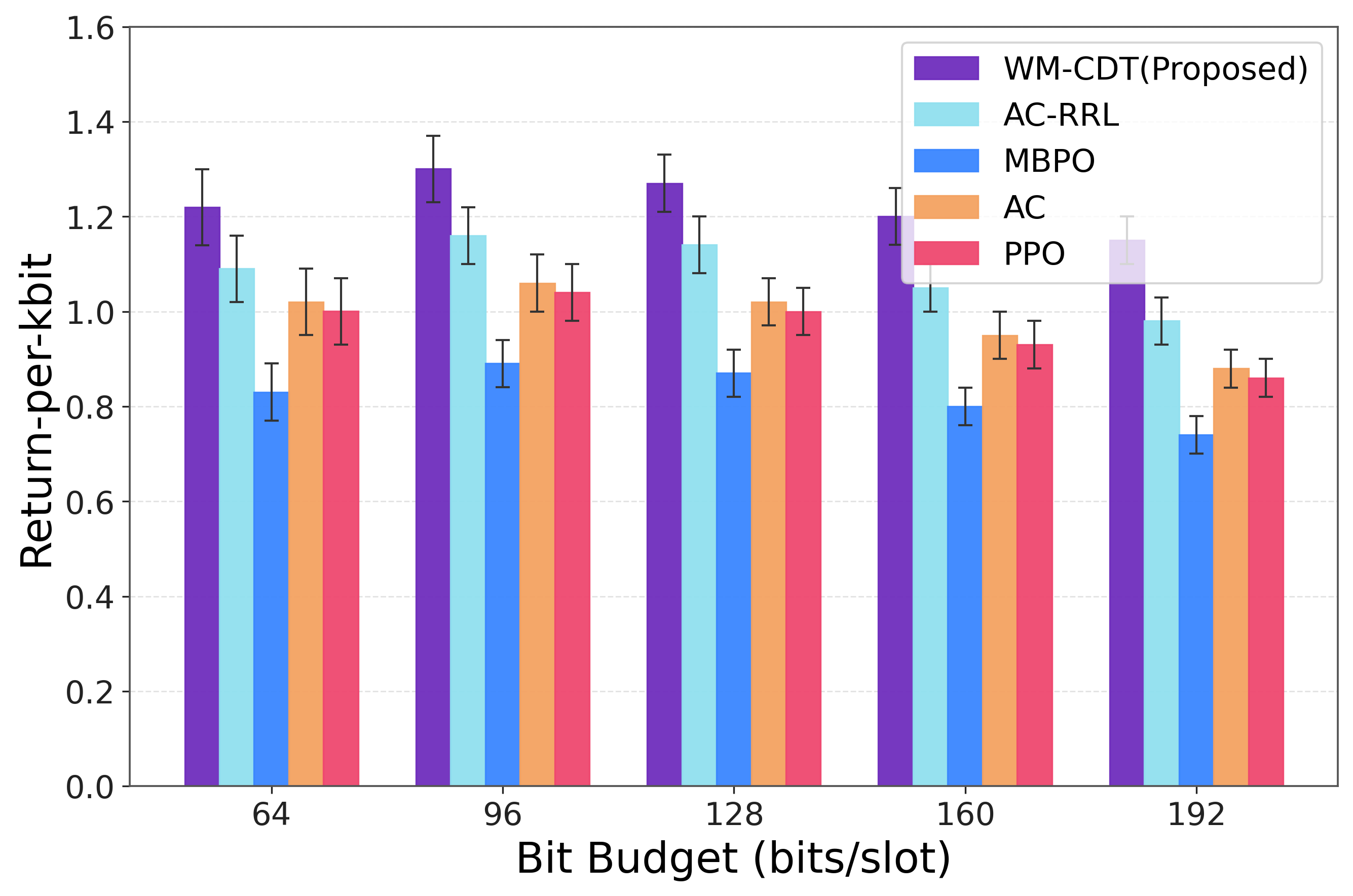}
    \vspace{-0.4cm}
  \caption{The return-per-kbit of different approaches under bit budgets.}
  \label{fig:return}
\end{figure}

\begin{figure}[t!]
  \centering
  \includegraphics[width=0.8\linewidth]{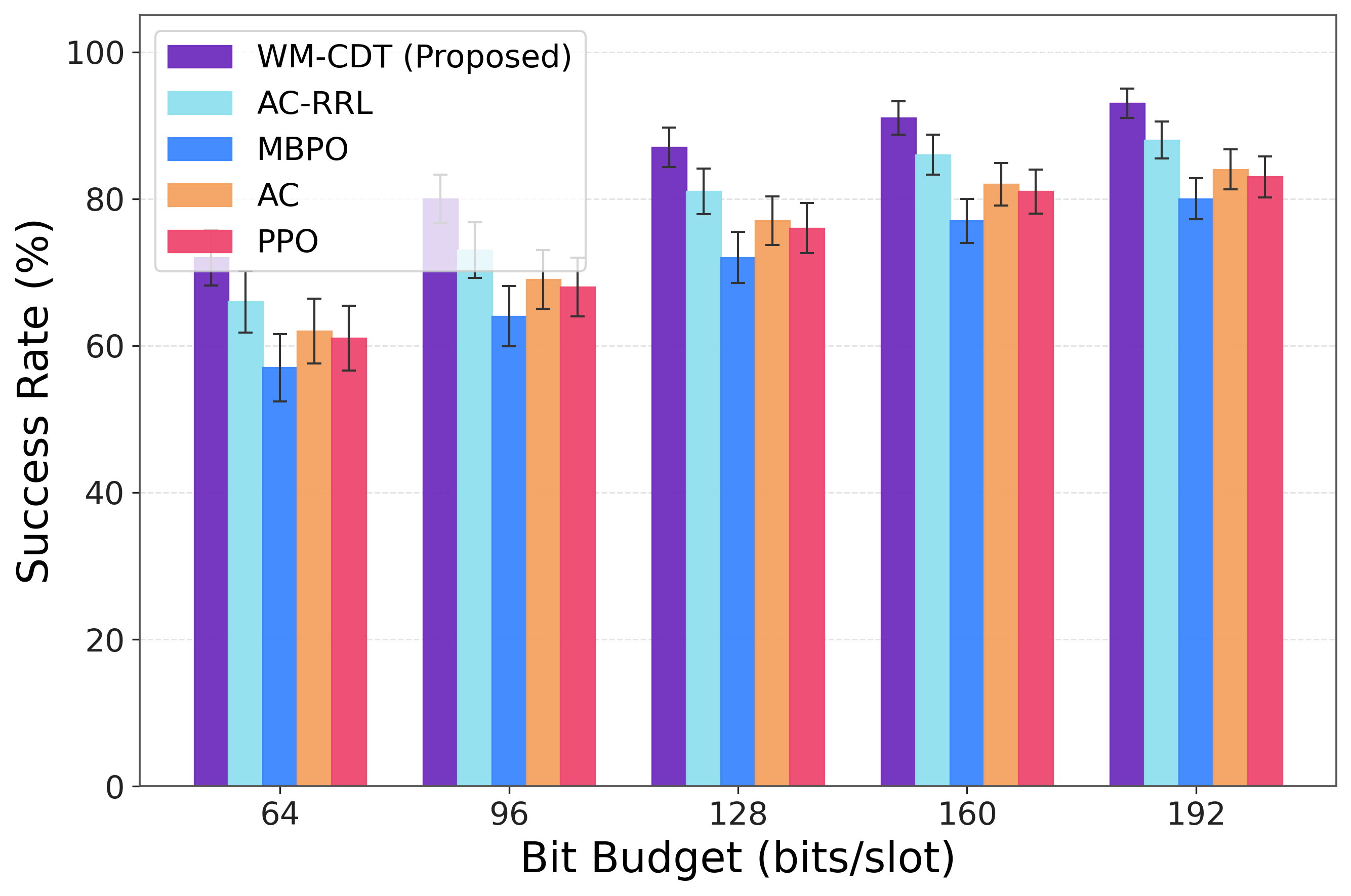}
    \vspace{-0.4cm}
  \caption{The navigation success rate of different approaches under bit budgets.}
      \vspace{-0.3cm}
  \label{fig:success}
\end{figure}

Fig.~\ref{fig:return} shows the return-per-kbit of different methods under different bit budgets, while Fig.~\ref{fig:success} shows the corresponding navigation success rate. From Fig.~\ref{fig:return}, the proposed WM-CDT scheme achieves the highest return-per-kbit across all bit budgets. Particularly, WM-CDT achieves up to $17.3\%$ higher return-per-kbit than AC-RRL, $55.4\%$ higher than MBPO, $30.7\%$ higher than AC, and $33.7\%$ higher than PPO. It shows that the WM-CDT approach can more efficiently convert transmitted semantic tokens into long-horizon task return.
Fig.~\ref{fig:success} further shows that WM-CDT can achieve both high communication efficiency and navigation success rate. As the bit budget increases, the success rate of all methods improves because more semantic tokens can be transmitted to support obstacle awareness and goal-directed control. It is observed that WM-CDT achieves the highest success rate. Compared with AC-RRL, MBPO, AC, and PPO, WM-CDT achieves up to $9.6\%$, $26.3\%$, $16.1\%$, and $18.0\%$ improvement in navigation success rate, respectively, which indicates that WM-CDT improves both semantic communication efficiency and closed-loop task performance.

Figs.~\ref{fig:return} and~\ref{fig:success} jointly show the tradeoff between task performance and communication efficiency. In Fig.~\ref{fig:success}, the success rate increases with the bit budget since more semantic tokens provide the edge server with richer task-related information. In contrast, Fig.~\ref{fig:return} shows that the return-per-kbit reaches its maximum around $96$ bits/slot and then gradually decreases as the bit budget increases. This is because, after the most task-critical tokens have been transmitted, additional tokens provide diminishing long-horizon control benefit while still increasing the communication cost.
It can be also observed that AC-RRL outperforms AC since the recurrent state allows the edge server to accumulate historical semantic information under partial observability. However, AC-RRL still optimizes semantic selection directly from real interactions and cannot explicitly evaluate the counterfactual long-horizon contribution of each token. In contrast, WM-CDT uses the learned world model to perform imagined rollouts and estimate token-level CIV, thereby suppressing redundant tokens and preserving those that are more useful for future control. MBPO also uses model-based rollouts, but it lacks the RSSM-based belief update and counterfactual token valuation of WM-CDT, which limits its performance under coupled dynamics of token selection, wireless transmission, and UAV control.

\begin{figure}[t!]
  \centering
   \vspace{-0.3cm}
  \includegraphics[width=0.85\linewidth]{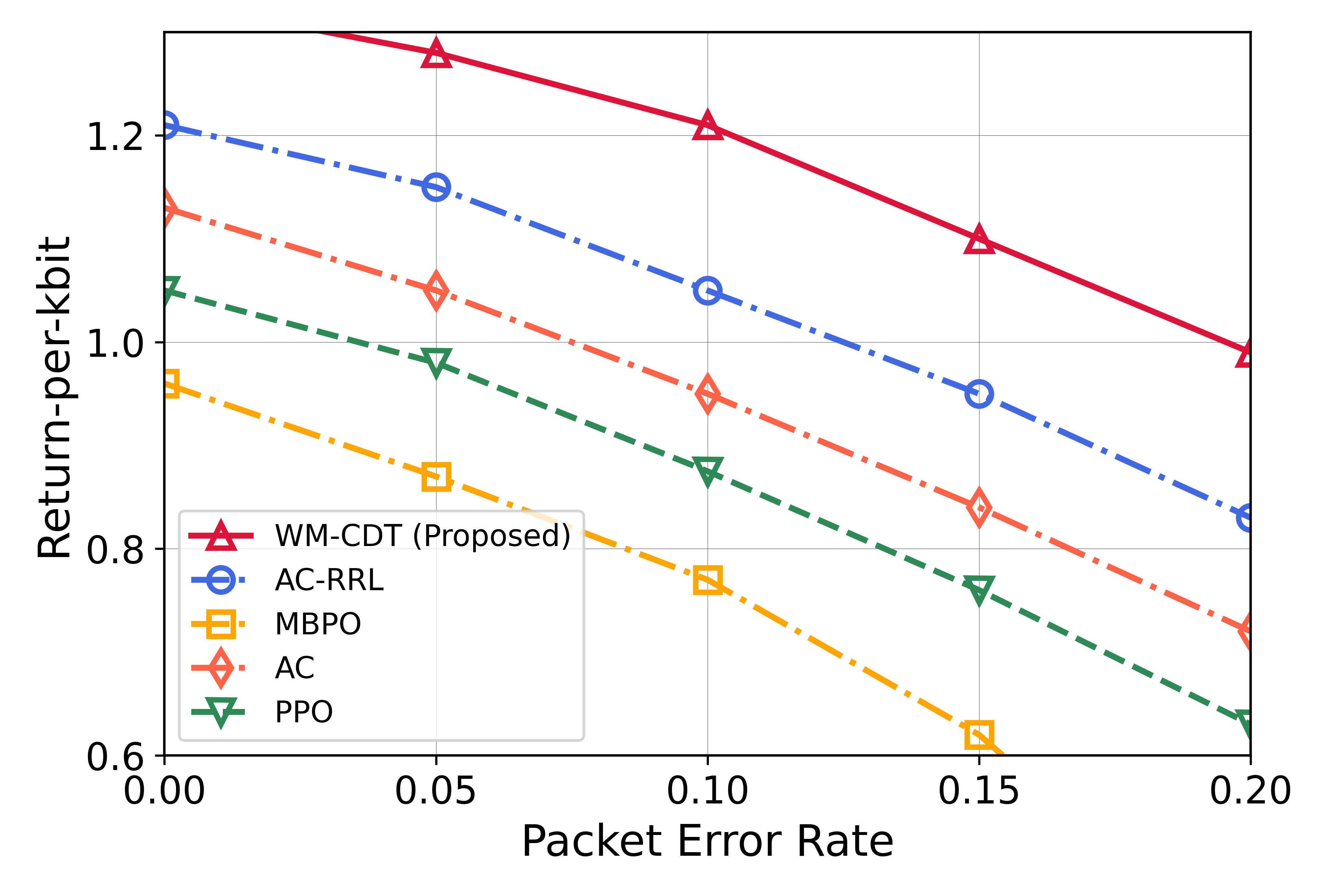}
    \vspace{-0.45cm}
  \caption{Return-per-kbit resulting from the studied schemes under different packet error rates.}
  \label{fig:per}
\end{figure}
Fig.~\ref{fig:per} shows the robustness of different approaches under packet errors. As the packet error rate increases, the return-per-kbit of all methods decreases since some transmitted semantic tokens are corrupted or lost. We can observe that WM-CDT achieves up to $19.6\%$, $77.4\%$, $46.2\%$, and $32.5\%$ improvements in return-per-kbit compared to AC-RRL, MBPO, AC, and PPO, respectively, under $20\%$ packet error rate.
This gain comes from two factors. First, the RSSM-based belief state allows WM-CDT to compensate for missing tokens by using historical semantic information. Second, the CIV-guided selector prioritizes tokens with larger long-horizon causal impact, thus making the transmitted message less redundant and more control-relevant. AC-RRL is more robust than AC and PPO due to its recurrent memory, but it lacks counterfactual token valuation. MBPO also degrades under packet errors because its dynamics model is sensitive to corrupted semantic observations. In contrast, WM-CDT combines belief prediction and causal token selection, thus, it has stronger robustness under unreliable wireless transmission.

\vspace{-0.4cm}
\subsection{CIV Effectiveness}
\vspace{-0.1cm}
\begin{figure}[t!]
  \centering
  \includegraphics[width=0.8\linewidth]{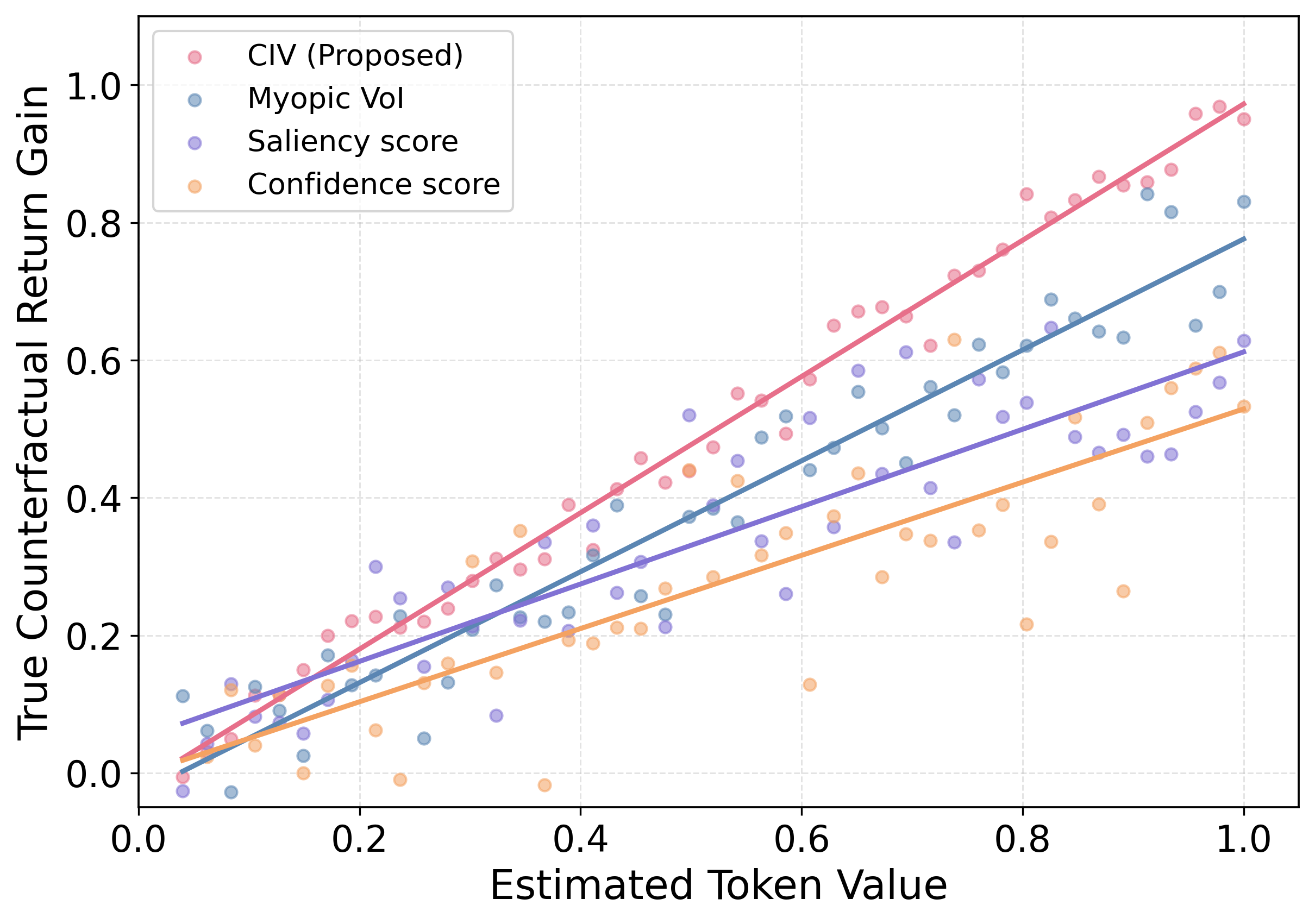}
    \vspace{-0.4cm}
  \caption{Correlation between different token-importance metrics and the true counterfactual return gain.}
  \label{fig:civ}
\end{figure}
Fig.~\ref{fig:civ} compares different token-importance metrics against the true counterfactual return gain. The proposed CIV metric shows the strongest positive correlation, and its fitted line is closest to the ideal monotonic trend. In other words, Fig.~\ref{fig:civ} shows that CIV can better identify semantic tokens that have large long-horizon impact on future control performance.
In contrast, myopic VoI, saliency score, and confidence score show weaker correlations. Myopic VoI captures only short-term value changes, while saliency and confidence mainly reflect instantaneous visual or detection relevance. These metrics may select tokens that appear important at the current time step but have limited effect on long-horizon navigation. Hence, Fig.~\ref{fig:civ} shows the effectiveness of counterfactual CIV estimation for communication-efficient token selection in closed-loop physical AI systems.

\begin{figure}[t!]
  \centering
  \includegraphics[width=1\linewidth]{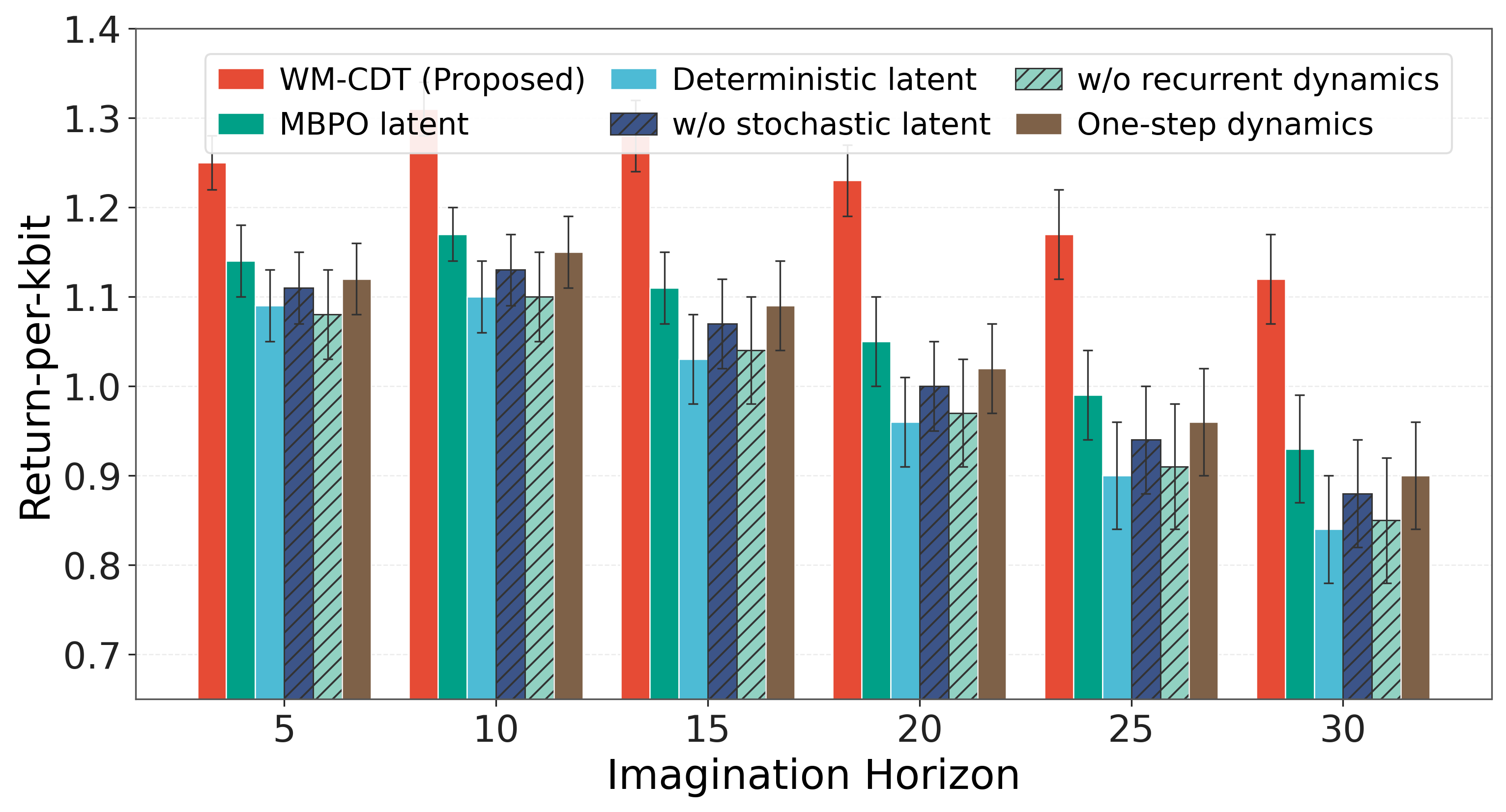}
  \vspace{-0.8cm}
\caption{Ablation study of the latent dynamics design under different imagination horizons.}
  \vspace{-0.2cm}
  \label{fig:dynamics}
\end{figure}

\vspace{-0.4cm}
\subsection{Ablation Studies}
\vspace{-0.1cm}
Figs.~\ref{fig:dynamics}-\ref{fig:selection} evaluate three key components of WM-CDT: latent dynamics modeling, semantic encoder design, and token-selection strategy. For the dynamics ablation in Fig.~\ref{fig:dynamics}, ``MBPO latent'' replaces the RSSM model with a latent dynamics model, ``deterministic latent''  removes probabilistic belief modeling, ``w/o stochastic latent''  removes the stochastic latent variable while keeping the recurrent transition, ``w/o recurrent dynamics'' removes the recurrent state transition, and ``one-step dynamics'' only uses one-step prediction without long-horizon latent rollout. For the encoder ablation in Fig.~\ref{fig:encoder}, ``reconstruction-based'' uses reconstruction-oriented representation learning, ``w/o TD prediction'' removes the TD latent prediction objective, ``w/o control-conditioned'' removes action-conditioned prediction, and ``frozen encoder'' keeps the encoder fixed during training. For the selection ablation in Fig.~\ref{fig:selection}, ``w/o CIV'' removes CIV estimation, ``myopic VoI'' uses only short-term value change, ``saliency Top-$K$'' and ``confidence Top-$K$'' select tokens by instantaneous saliency or confidence, ``random'' selects feasible tokens randomly, ``w/o per-bit norm'' uses CIV without bit normalization, ``w/o reverse pruning'' removes subset-conditioned pruning, and ``fixed Top-$K$'' transmits a fixed number of tokens.

Fig.~\ref{fig:dynamics} studies the impact of the latent dynamics design. The proposed method achieves the highest return-per-kbit over all imagination horizons. At the best horizon, the proposed method achieves about $12.0\%$ higher return-per-kbit than the MBPO-style latent model, and $13.9\%$, $15.9\%$, and $19.1\%$ higher return-per-kbit than the one-step dynamics, w/o stochastic latent, and deterministic latent variants, respectively. These results show that both the stochastic latent state and the recurrent dynamics update are important for modeling the uncertainty and temporal dependency in semantic communication and control. In contrast, one-step dynamics and deterministic latent transitions are less effective in capturing the long-horizon closed-loop evolution.

\begin{figure}[t!]
  \centering
    \vspace{-0.3cm}
  \includegraphics[width=1\linewidth]{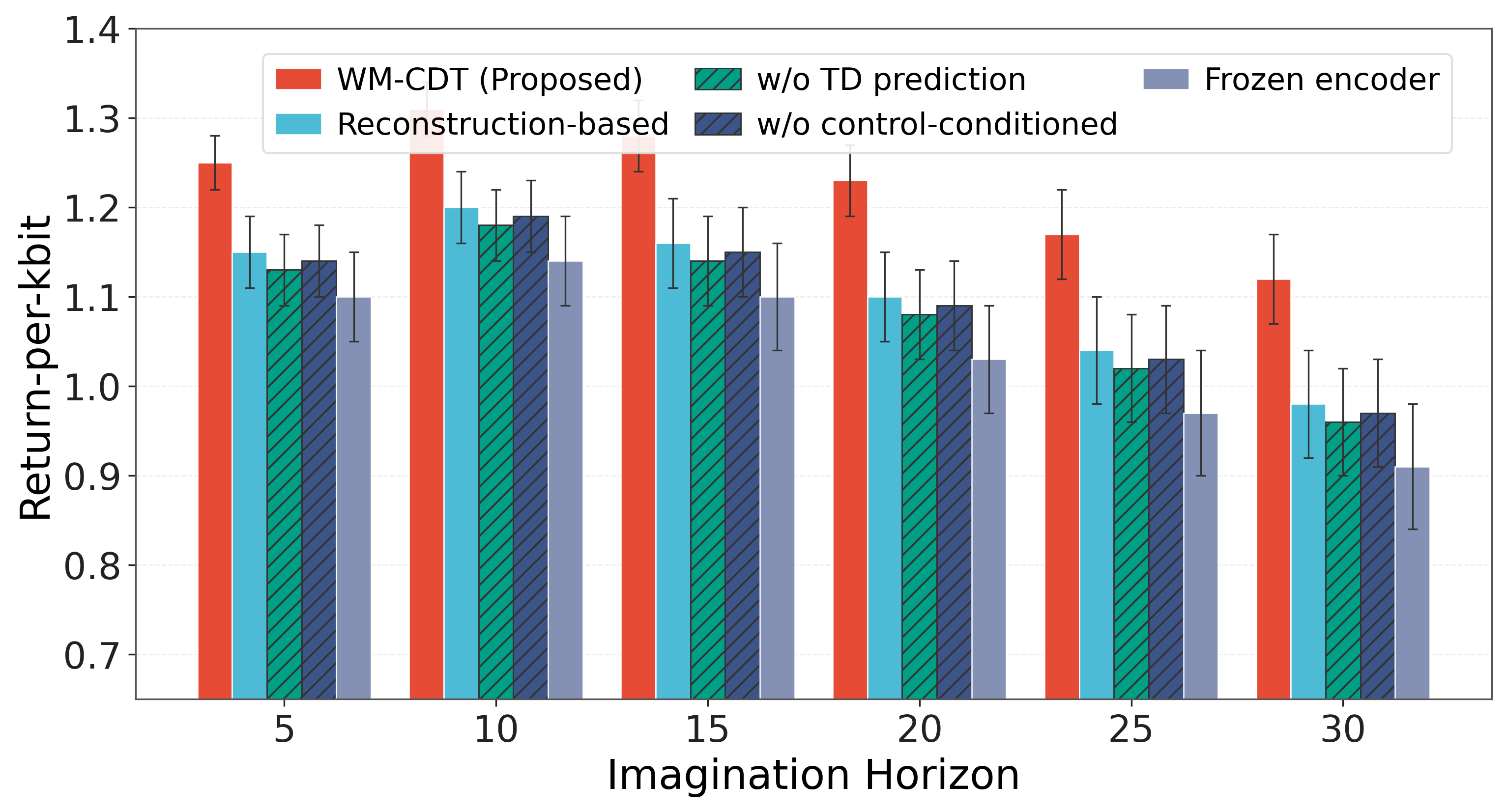}
  \vspace{-0.8cm}
\caption{Ablation study of the semantic encoder design under different imagination horizons.}
  \label{fig:encoder}
\end{figure}

\begin{figure}[t!]
  \centering
    \vspace{-0.3cm}
  \includegraphics[width=1\linewidth]{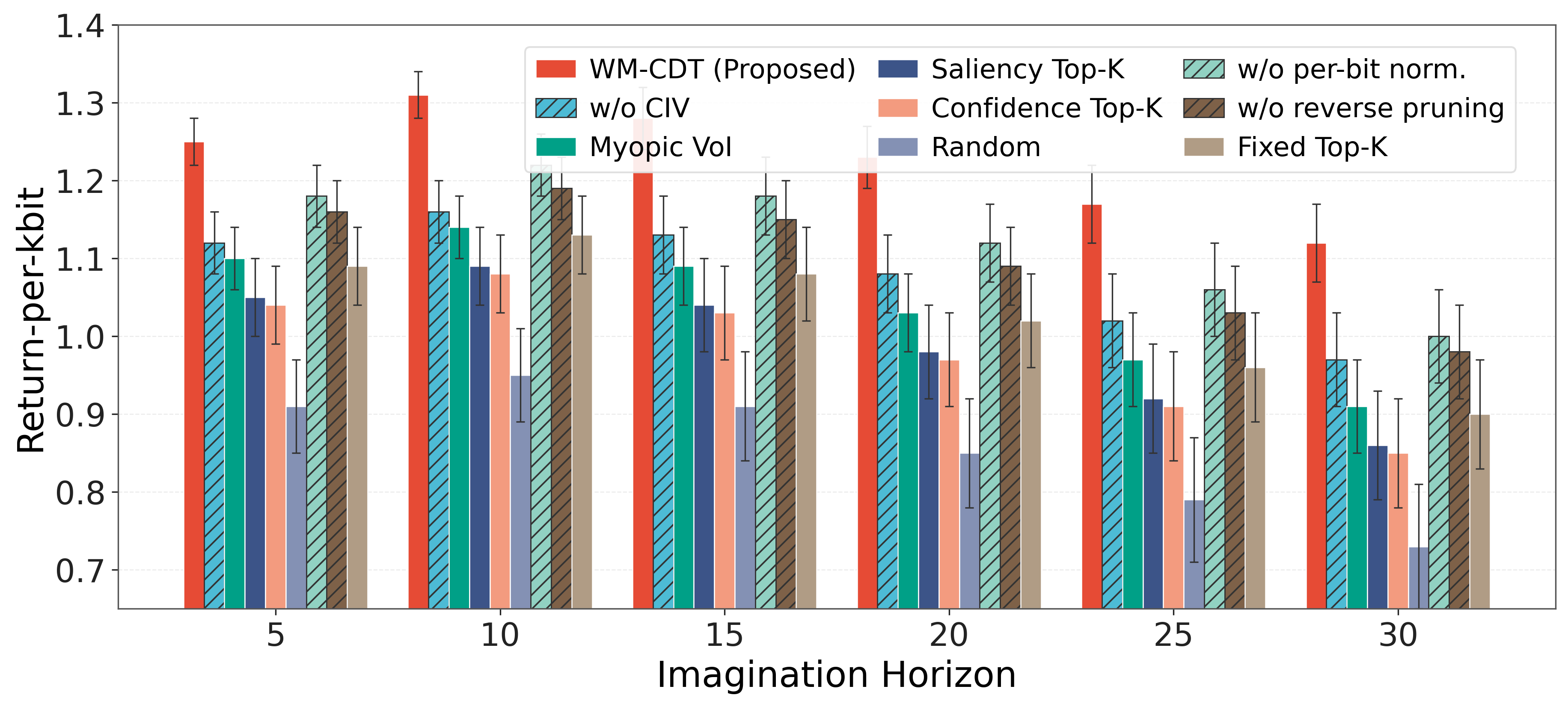}
  \vspace{-0.8cm}
\caption{Ablation study of token-selection strategies under different imagination horizons.}
  \label{fig:selection}
\end{figure}

Fig.~\ref{fig:encoder} evaluates different semantic encoder designs. The proposed TD-based encoder outperforms the reconstruction-based encoder, the w/o TD prediction variant, the w/o control-conditioned variant, and the frozen encoder. At horizon $H=10$, the proposed method improves the return-per-kbit by about $9.2\%$ over the reconstruction-based encoder and by $10.1\%$ over the w/o control-conditioned variant. The gap becomes larger at long horizons, where the proposed encoder achieves up to $23.1\%$ improvement over the frozen encoder. It indicates that temporally predictive and control-aware semantic representations are more suitable than purely reconstruction-oriented or static features for long-horizon semantic control.

In Fig.~\ref{fig:selection}, we show return-per-kbit versus different token-selection mechanisms. The proposed CIV-based selector achieves the highest return-per-kbit across all imagination horizons. At horizon $H=10$, it achieves up to $7.4\%$, $10.1\%$, $12.9\%$, $14.9\%$ improvements in return-per-kbit compared to the w/o per-bit normalization variant,  the w/o reverse pruning variant, the w/o CIV variant, and the myopic VoI baseline. The proposed CIV-based selector also shows significant improvements in return-per-kbit compared to saliency-based, confidence-based, fixed Top-$K$, and random selection. This is because counterfactual CIV estimation of the proposed approach identifies tokens with high long-horizon contribution, per-bit normalization avoids bias toward long tokens, and reverse pruning improves subset refinement under the bit budget. Overall, Figs.~\ref{fig:dynamics}-\ref{fig:selection} show that the gains of the proposed WM-CDT approach come from the joint effect of the RSSM-based dynamics model, the TD-based semantic encoder, and the CIV-guided budget-aware token selector.

\vspace{-0.4cm}
\section{Conclusion}
\vspace{-0.1cm}
In this paper, we have studied goal-oriented semantic communications for closed-loop physical AI systems, where transmitted semantic tokens affect not only current inference but also future belief updates, control actions, and long-horizon task performance. We have formulated the problem as a return-per-bit maximization under wireless bit-budget constraints, thereby jointly capturing task efficiency and communication efficiency. To evaluate the long-horizon utility of semantic tokens, we have introduced a CIV metric based on counterfactual transmission interventions. We then have proposed a WM-CDT framework, which learns an RSSM-based latent model at the edge server, performs imagined rollouts for long-horizon control, and guides semantic token selection through CIV-based counterfactual reasoning.
Extensive simulation results on the AirSim-Sionna-based UAV navigation environment have shown that the proposed WM-CDT  achieves up to $17.3\%$, $55.4\%$, $30.7\%$, and $33.7\%$ higher return-per-kbit than AC-RRL, MBPO, AC, and PPO, respectively, and improves the navigation success rate by up to $9.6\%$, $26.3\%$, $16.1\%$, and $18.0\%$. Moreover, simulation results have also shown that WM-CDT achieves the highest return-per-kbit under unreliable wireless transmission, and counterfactual CIV is more aligned with the true long-horizon return gain than myopic VoI, saliency score, and confidence score.

\vspace{-0.2cm}
\section*{Appendix}
\vspace{-0.1cm}
\section*{Proof of Theorem 1}
For any policy $\pi\in\Pi$ and any scalar $\alpha\in\mathbb{R}$, we define the $\alpha$-priced objective gap $
\Delta_\alpha(\pi)\triangleq R(\pi)-\alpha C(\pi)$,
where $R(\pi)=\mathbb{E}_{\pi}\!\left[\sum_{t=0}^{T}\gamma^{t} r_t\right]$ and
$C(\pi)=\mathbb{E}_{\pi}\!\left[\sum_{t=0}^{T}\gamma^{t} \ell(c_t)\right]$.
By assumption, $C(\pi)>0$ for all $\pi\in\Pi$, and, thus, we can divide by $C(\pi)$ and rewrite $\Delta_\alpha(\pi)$ exactly as
\vspace{-0.2cm}
\begin{equation}
\begin{aligned}
\Delta_\alpha(\pi)
=R(\pi)-\alpha C(\pi)
=C(\pi)\big(\eta(\pi)-\alpha\big),
\label{eq:delta_factor_detailed}
\end{aligned}
\end{equation}
where $\eta(\pi)\triangleq R(\pi)/C(\pi)$. Since $C(\pi)>0$, \eqref{eq:delta_factor_detailed} implies that
\vspace{-0.2cm}
\begin{equation}
\begin{aligned}
&\Delta_\alpha(\pi) > 0\ \Longleftrightarrow\ \eta(\pi)>\alpha,\\
&\Delta_\alpha(\pi) = 0\ \Longleftrightarrow\ \eta(\pi)=\alpha,\\
&\Delta_\alpha(\pi) < 0\ \Longleftrightarrow\ \eta(\pi)<\alpha.
\label{eq:sign_equiv}
\end{aligned}
\end{equation}
We define $F(\alpha)\triangleq \max_{\pi\in\Pi}\Delta_\alpha(\pi)=\max_{\pi\in\Pi}\big(R(\pi)-\alpha C(\pi)\big)$.
We can now prove the theorem statements.

First, we prove the sign of $F(\alpha)$ for $\alpha$ below and above $\eta^*$. Let $\eta^*\triangleq \max_{\pi\in\Pi}\eta(\pi)$. There are two cases.

\emph{Case 1: $\alpha<\eta^*$.} By the definition of maximum, $\alpha<\max_{\pi}\eta(\pi)$ means that there exists at least one policy
$\bar\pi\in\Pi$ such that $\eta(\bar\pi)>\alpha$.
Applying \eqref{eq:sign_equiv} yields $\Delta_\alpha(\bar\pi)>0$.
Since $F(\alpha)$ is the maximum over all $\Delta_\alpha(\pi)$, we have
\vspace{-0.2cm}
\begin{equation}
F(\alpha)=\max_{\pi\in\Pi}\Delta_\alpha(\pi)\ \ge\ \Delta_\alpha(\bar\pi)\ >\ 0.
\label{eq:F_pos}
\end{equation}

\emph{Case 2: $\alpha>\eta^*$.} For every $\pi\in\Pi$, we have $\eta(\pi)\le \eta^*<\alpha$.
Using \eqref{eq:sign_equiv} again gives $\Delta_\alpha(\pi)<0$ for every $\pi\in\Pi$.
Since $F(\alpha)$ is the maximum over all $\Delta_\alpha(\pi)$, we have
\vspace{-0.2cm}
\begin{equation}
F(\alpha)=\max_{\pi\in\Pi}\Delta_\alpha(\pi)\ <\ 0.
\label{eq:F_neg}
\end{equation}
\vspace{-0.4cm}

Then, we prove the value of $F(\alpha)$ at $\alpha=\eta^*$.

\emph{Upper bound.} We first show $F(\eta^*)=0$. Since $\eta(\pi)\le \eta^*$ for all $\pi\in\Pi$, \eqref{eq:sign_equiv} implies
$\Delta_{\eta^*}(\pi)\le 0$ for all $\pi\in\Pi$. Taking the maximum over $\pi$ preserves the inequality:
\vspace{-0.2cm}
\begin{equation}
F(\eta^*)=\max_{\pi\in\Pi}\Delta_{\eta^*}(\pi)\ \le\ 0.
\label{eq:F_le0}
\end{equation}
\vspace{-0.4cm}

\emph{Lower bound.} By the definition of $\eta^*=\max_{\pi\in\Pi}\eta(\pi)$, there exists a policy $\pi^\dagger\in\Pi$
such that $\eta(\pi^\dagger)=\eta^*$.
We substitute $\pi^\dagger$ into \eqref{eq:delta_factor_detailed} at $\alpha=\eta^*$:
$
\Delta_{\eta^*}(\pi^\dagger)=C(\pi^\dagger)\big(\eta(\pi^\dagger)-\eta^*\big)=0.
$
Since $F(\eta^*)$ is a maximum, it dominates each element of the set:
$F(\eta^*)\ge \Delta_{\eta^*}(\pi^\dagger)$. Hence,
\vspace{-0.2cm}
\begin{equation}
F(\eta^*)\ \ge\ \Delta_{\eta^*}(\pi^\dagger)=0.
\label{eq:F_ge0}
\end{equation}
Combining \eqref{eq:F_le0} and \eqref{eq:F_ge0}, then we have $F(\eta^*)=0$.

\emph{Uniqueness.} From \eqref{eq:F_pos} and \eqref{eq:F_neg}, $F(\alpha)>0$ for all $\alpha<\eta^*$ and $F(\alpha)<0$ for all $\alpha>\eta^*$.
Therefore no $\alpha\neq \eta^*$ can satisfy $F(\alpha)=0$, and $\eta^*$ is the unique root.

\emph{Ratio-optimal $\pi^*$.}
Assume $\pi^*$ is ratio-optimal, i.e., $\eta(\pi^*)=\eta^*$.
Substituting $\pi^*$ into \eqref{eq:delta_factor_detailed} with $\alpha=\eta^*$ gives
\vspace{-0.2cm}
\begin{equation}
\Delta_{\eta^*}(\pi^*)=C(\pi^*)\big(\eta(\pi^*)-\eta^*\big)=0.
\label{eq:opt_gap0}
\end{equation}
For any $\pi\in\Pi$, $\eta(\pi)\le \eta^*$ implies $\Delta_{\eta^*}(\pi)\le 0$ by \eqref{eq:sign_equiv}.
Hence $0=\Delta_{\eta^*}(\pi^*)\ge \Delta_{\eta^*}(\pi)$ for all $\pi$, meaning that $\pi^*$ maximizes $R(\pi)-\eta^* C(\pi)$ and attains the maximum value $0$.
Conversely, suppose $\pi^*\in\arg\max_{\pi\in\Pi}\Delta_{\eta^*}(\pi)$ and $\Delta_{\eta^*}(\pi^*)=0$.
Then \eqref{eq:delta_factor_detailed} with $\alpha=\eta^*$ yields
$C(\pi^*)\big(\eta(\pi^*)-\eta^*\big)=0$. Since $C(\pi^*)>0$, we must have $\eta(\pi^*)=\eta^*$,
so $\pi^*$ is ratio-optimal.
Hence, Theorem 1 is proved.

\bibliography{IEEEabrv,reference}

\end{document}